\documentclass{article}

\PassOptionsToPackage{numbers,sort, compress}{natbib}

\usepackage[final]{neurips_2023}

\usepackage[utf8]{inputenc} 
\usepackage{xcolor}         
\usepackage[T1]{fontenc}    
\usepackage{hyperref}      
\usepackage{url}            
\usepackage{booktabs}       
\usepackage{amsfonts}       
\usepackage{nicefrac}       
\usepackage{microtype}      
\usepackage{wrapfig} 

\makeatletter
\renewcommand\thesubsubsection{\@Alph\c@section.\@arabic\c@subsubsection}
\makeatother

\usepackage{amsmath, amssymb, amsthm, bbm, mathtools, commath, amsfonts, tikz-cd}
\usepackage{multirow}
\usepackage{caption}
\usepackage{subcaption}
\usepackage{float}

\newtheorem{prop}{Proposition}

\newcommand{\R}{\mathbb{R}}

\newcommand{\M}{\mathcal{M}}

\newtheoremstyle{TheoremNum}
    {\topsep}{\topsep}              
    {\itshape}                      
    {}                              
    {\bfseries}                     
    {.}                             
    { }                             
    {\thmname{#1}\thmnote{ \bfseries #3}}
\theoremstyle{TheoremNum}

\newtheoremstyle{CorollaryNum}
    {\topsep}{\topsep}              
    {\itshape}                      
    {}                              
    {\bfseries}                     
    {.}                             
    { }                             
    {\thmname{#1}\thmnote{ \bfseries #3}}
\theoremstyle{CorollaryNum}

\newtheoremstyle{LemmaNum}
    {\topsep}{\topsep}              
    {\itshape}                      
    {}                              
    {\bfseries}                     
    {.}                             
    { }                             
    {\thmname{#1}\thmnote{ \bfseries #3}}
\theoremstyle{LemmaNum}

\usepackage{algorithm}

\usepackage{enumitem}
\usepackage{xkcdcolors}

\title{Riemannian Residual Neural Networks}

\author{%
  Isay Katsman$^*$ \\
  Yale University \\
  \texttt{isay.katsman@yale.edu} \\
  \And
  Eric M. Chen$^*$, Sidhanth Holalkere$^*$\\
  Cornell University\\
  \texttt{\{emc348, sh844\}@cornell.edu}\\
  \And
  Anna Asch\\
  Cornell University \\ 
  \texttt{aca89@cornell.edu} \\
  \AND
  Aaron Lou\\
  Stanford University \\
  \texttt{aaronlou@stanford.edu}\\
  \And
  Ser-Nam Lim$^\dagger$ \\
  University of Central Florida\\
  \texttt{sernam@ucf.edu} \\
  \And
  Christopher De Sa\\
  Cornell University\\
  \texttt{cdesa@cs.cornell.edu}
}

\begin{document}

\maketitle

\begin{abstract}
    Recent methods in geometric deep learning have introduced various neural networks to operate over data that lie on Riemannian manifolds. Such networks are often necessary to learn well over graphs with a hierarchical structure or to learn over manifold-valued data encountered in the natural sciences. These networks are often inspired by and directly generalize standard Euclidean neural networks. However, extending Euclidean networks is difficult and has only been done for a select few manifolds. In this work, we examine the residual neural network (ResNet) and show how to extend this construction to general Riemannian manifolds in a geometrically principled manner. Originally introduced to help solve the vanishing gradient problem, ResNets have become ubiquitous in machine learning due to their beneficial learning properties, excellent empirical results, and easy-to-incorporate nature when building varied neural networks. We find that our Riemannian ResNets mirror these desirable properties: when compared to existing manifold neural networks designed to learn over hyperbolic space and the manifold of symmetric positive definite matrices, we outperform both kinds of networks in terms of relevant testing metrics and training dynamics.
\end{abstract}

\section{Introduction}
\begin{wrapfigure}{R}{.45\textwidth}
    \vspace{-5pt}
\centering
  \includegraphics[width=0.35\textwidth,page=2]{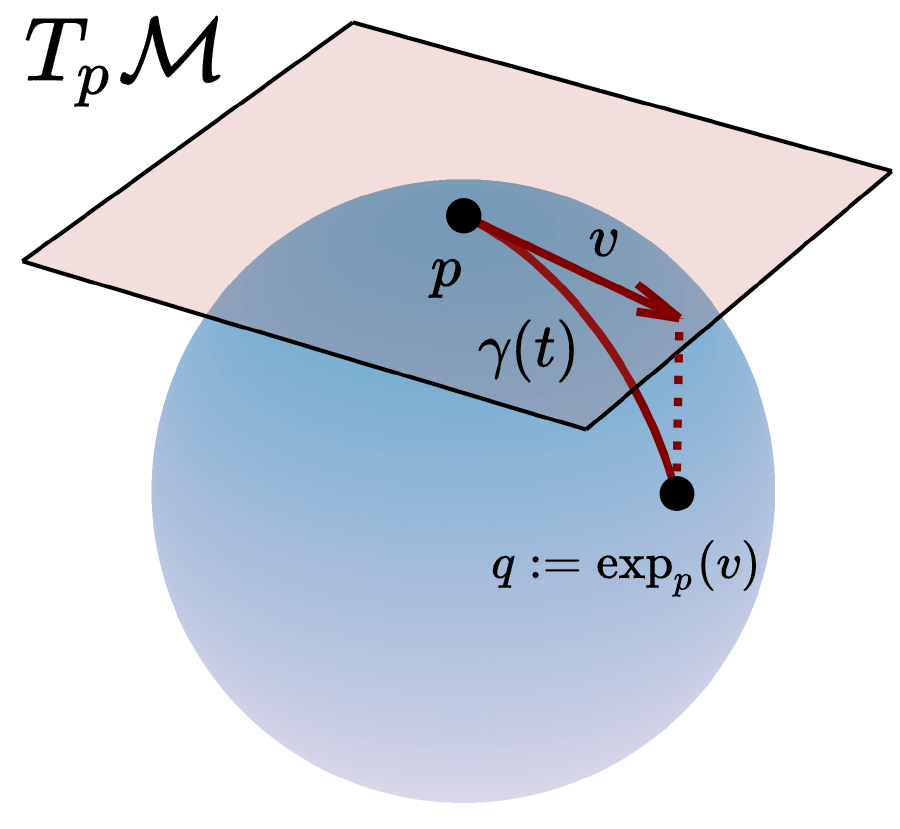}
  \caption{An illustration of a manifold-generalized residual addition. The traditional Euclidean formula $p \leftarrow p + v$ is generalized to $p \leftarrow \exp_p (v)$, where $\exp$ is the Riemannian exponential map. $\M$ is the manifold and $T_p \M$ is the tangent space at $p$.}
  \label{fig:intro_fig}
  \vspace{-10pt}
\end{wrapfigure}

In machine learning, it is common to represent data as vectors in Euclidean space (i.e. $\mathbb{R}^n$). The primary reason for such a choice is convenience, as this space has a classical vectorial structure, a closed-form distance formula, and a simple inner-product computation. Moreover, the myriad existing Euclidean neural network constructions enable performant learning.
{\let\thefootnote\relax\footnotetext{* indicates equal contribution.}}

Despite the ubiquity and success of Euclidean embeddings, recent research \citep{nickel2017poincare} has brought attention to the fact that several kinds of complex data require manifold considerations. Such data are various and range from covariance matrices, represented as points on the manifold of symmetric positive definite (SPD) matrices \citep{Huang2017ARN},
to angular orientations, represented as points on tori, found in the context of robotics \citep{rezende2020normalizing}. However, generalizing Euclidean neural network tools to manifold structures such as these can be quite difficult in practice. Most prior works design network architectures for a specific manifold \citep{ganea2018hyperbolic, cohen2018spherical}, thereby inefficiently necessitating a specific design for each new manifold.
{\let\thefootnote\relax\footnotetext{$^\dagger$Work done while at Meta AI.}}

We address this issue by extending Residual Neural Networks \citep{He2016DeepRL} to Riemannian manifolds in a way that naturally captures the underlying geometry. We construct our network by parameterizing vector fields and leveraging geodesic structure (provided by the Riemannian $\exp$ map) to ``add" the learned vectors to the input points, thereby naturally generalizing a typical Euclidean residual addition. This process is illustrated in Figure \ref{fig:intro_fig}. Note that this strategy is exceptionally natural, only making use of inherent geodesic geometry, and works generally for all smooth manifolds. We refer to such networks as Riemannian residual neural networks.

Though the above approach is principled, it is underspecified, as constructing an efficient learnable vector field for a given manifold is often nontrivial. To resolve this issue, we present a general way to induce a learnable vector field for a manifold $\M$ given only a map $f : \M \rightarrow \mathbb{R}^k$. Ideally, this map should capture intrinsic manifold geometry. For example, in the context of Euclidean space, this map could consist of a series of $k$ projections onto hyperplanes. There is a natural equivalent of this in hyperbolic space that instead projects to horospheres (horospheres correspond to hyperplanes in Euclidean space). More generally, we propose a feature map that once more relies only on geodesic information, consisting of projection to random (or learned) geodesic balls. This final approach provides a fully geometric way to construct vector fields, and therefore natural residual networks, for any Riemannian manifold.

After introducing our general theory, we give concrete manifestations of vector fields, and therefore residual neural networks, for hyperbolic space and the manifold of SPD matrices. We compare the performance of our Riemannian residual neural networks to that of existing manifold-specific networks on hyperbolic space and on the manifold of SPD matrices, showing that our networks perform much better in terms of relevant metrics due to their improved adherence to manifold geometry.

Our contributions are as follows: 
\begin{enumerate}[itemsep=0.5mm,leftmargin=20pt]
    \vspace{-5pt}
    \item We introduce a novel and principled generalization of residual neural networks to general Riemannian manifolds. Our construction relies only on knowledge of geodesics, which capture manifold geometry.
    \item Theoretically, we show that our methodology better captures manifold geometry than pre-existing manifold-specific neural network constructions. Empirically, we apply our general construction to hyperbolic space and to the manifold of SPD matrices. On various hyperbolic graph datasets (where hyperbolicity is measured by Gromov $\delta$-hyperbolicity) our method considerably outperforms existing work on both link prediction and node classification tasks. On various SPD covariance matrix classification datasets, a similar conclusion holds.
    \item Our method provides a way to directly vary the geometry of a given neural network without having to construct particular operations on a per-manifold basis. This provides the novel capability to directly compare the effect of geometric representation (in particular, evaluating the difference between a given Riemannian manifold $(\M, g)$ and Euclidean space $(\mathbb{R}^n,||\cdot||_2)$) while fixing the network architecture.
\end{enumerate}



\vspace{-5pt}
\section{Related Work}
\vspace{-5pt}
Our work is related to but distinctly different from existing neural ordinary differential equation (ODE) \cite{chen2018neural} literature as well a series of papers that have attempted generalizations of neural networks to specific manifolds such as hyperbolic space \citep{ganea2018hyperbolic} and the manifold of SPD matrices \citep{Huang2017ARN}.

\vspace{-3pt}
\subsection{Residual Networks and Neural ODEs}
\vspace{-3pt}

Residual networks (ResNets) were originally developed to enable training of larger networks, previously prone to vanishing and exploding gradients \citep{He2016DeepRL}. Later on, many discovered that by adding a learned residual, ResNets are similar to Euler's method \citep{Weinan2017APO, Lu2017BeyondFL, Haber2017StableAF, Ruthotto2018DeepNN, chen2018neural}. More specifically, the ResNet represented by $\textbf{h}_{t+1} = \textbf{h}_{t} + f(\textbf{h}, \theta_t)$ for $\textbf{h}_t \in \mathbb{R}^D$ mimics the dynamics of the ODE defined by $\frac{d \textbf{h}(t)}{dt} = f(\textbf{h}(t), t, \theta)$. Neural ODEs are defined precisely as ODEs of this form, where the local dynamics are given by a parameterized neural network. Similar to our work, \citet{lou2020neural, Katsman2021EquivariantMF, Falorsi2020NeuralOD, mathieu2020riemannian} generalize neural ODEs to Riemannian manifolds (further generalizing manifold-specific work such as \citet{bose2020latent}, that does this for hyperbolic space). However, instead of using a manifold's vector fields to solve a neural ODE, we learn an objective by parameterizing the vector fields directly (Figure \ref{fig:teaser2}). Neural ODEs and their generalizations to manifolds parameterize a continuous collection of vector fields over time for a single manifold in a dynamic flow-like construction. Our method instead parameterizes a discrete collection of vector fields, entirely untethered from any notion of solving an ODE. This makes our construction a strict generalization of both neural ODEs and their manifold equivalents \cite{lou2020neural, Katsman2021EquivariantMF, Falorsi2020NeuralOD, mathieu2020riemannian}.

\subsection{Riemannian Neural Networks}

Past literature has attempted generalizations of Euclidean neural networks to a number of manifolds.

\textbf{Hyperbolic Space} \citet{ganea2018hyperbolic} extended basic neural network operations (e.g. activation function, linear layer, recurrent architectures) to conform with the geometry of hyperbolic space through gyrovector constructions \citep{Ungar2009AGS}. In particular, they use gyrovector constructions \citep{Ungar2009AGS} to build analogues of activation functions, linear layers, and recurrent architectures. Building on this approach, \citet{chami2019hyperbolic} adapt these constructions to hyperbolic versions of the feature transformation and neighborhood aggregation steps found in message passing neural networks. Additionally, batch normalization for hyperbolic space was introduced in \citet{lou2020differentiating}; hyperbolic attention network equivalents were introduced in \citet{Glehre2019HyperbolicAN}. Although gyrovector constructions are algebraic and allow for generalization of neural network operations to hyperbolic space and beyond, we note that they do not capture intrinsic geodesic geometry. In particular, we note that the gyrovector-based hyperbolic linear layer introduced in \citet{ganea2018hyperbolic} reduces to a Euclidean matrix multiplication followed by a learned hyperbolic bias addition (see Appendix \ref{sec:hnnred}). Hence all non-Euclidean learning for this case happens through the bias term.
In an attempt to resolve this, further work has focused on imbuing these neural networks with more hyperbolic functions \citep{shimizu2020hyperbolic, chen-etal-2022-fully}. \citet{chen-etal-2022-fully} notably constructs a hyperbolic residual layer by projecting an output onto the Lorentzian manifold. However, we emphasize that our construction is more general while being more geometrically principled as we work with fundamental manifold operations like the exponential map rather than relying on the niceties of Lorentz space.

\citet{Yu2022HyLaHL} make use of randomized hyperbolic Laplacian features to learn in hyperbolic space. We note that the features learned are shallow and are constructed from a specific manifestation of the Laplace-Beltrami operator for hyperbolic space. In contrast, our method is general and enables non-shallow (i.e., multi-layer) feature learning.

\textbf{SPD Manifold} Neural network constructs have been extended to the manifold of symmetric positive definite (SPD) matrices as well.
In particular, SPDNet \citep{Huang2017ARN} is an example of a widely adopted SPD manifold neural network which introduced SPD-specific layers analogous to Euclidean linear and ReLU layers. Building upon SPDNet, \citet{Brooks2019RiemannianBN} developed a batch normalization method to be used with SPD data. Additionally, \citet{Lopez2021VectorvaluedDA} adapted gyrocalculus constructions used in hyperbolic space to the SPD manifold.

\textbf{Symmetric Spaces} Further work attempts generalization to symmetric spaces. \citet{Sonoda2022FullyConnectedNO} design fully-connected networks over noncompact symmetric spaces using particular theory from Helgason-Fourier analysis \cite{Helgason1965RadonFourierTO}, and \citet{Chakraborty2018ManifoldNetAD} attempt to generalize several operations such as convolution to such spaces by adapting and developing 
a weighted Fr\'echet mean construction. We note that the Helgason-Fourier construction in \citet{Sonoda2022FullyConnectedNO} exploits a fairly particular structure, while the weighted Fr\'echet mean construction in \citet{Chakraborty2018ManifoldNetAD} is specifically introduced for convolution, which is not the focus of our work (we focus on residual connections).


Unlike any of the manifold-specific work described above, our residual network construction can be applied generally to any smooth manifold and is constructed solely from geodesic information.



\vspace{-5pt}
\section{Background}
\vspace{-5pt}
\begin{figure*}[t]
\centering
\includegraphics[width=0.8\textwidth, page=1]{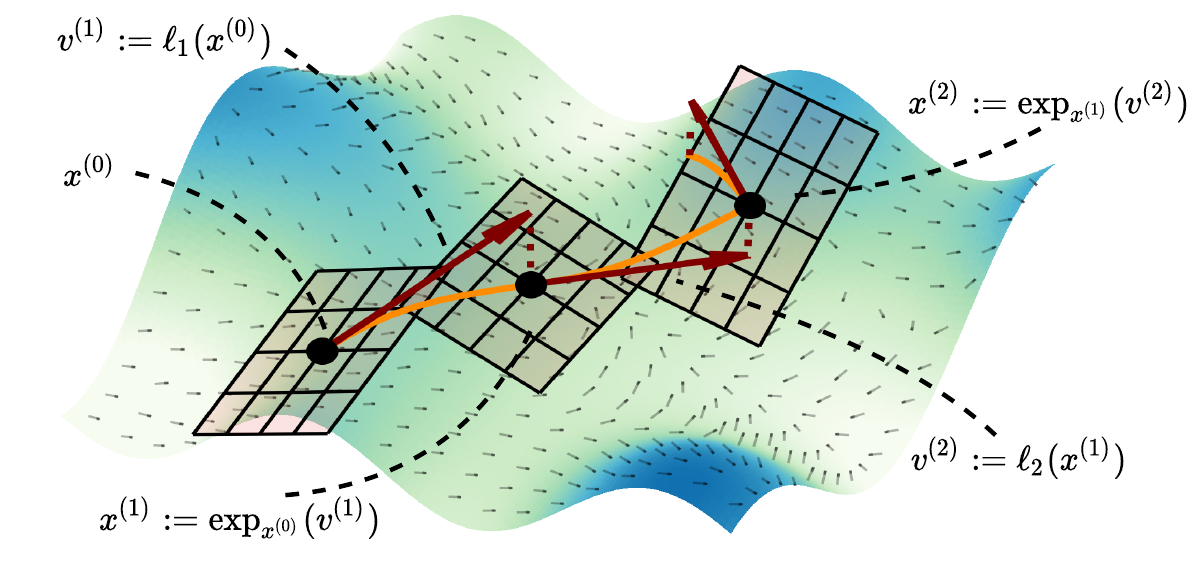}

\caption{A visualization of a Riemannian residual neural network on a manifold $\M$. Our model parameterizes vector fields on a manifold. At each layer in our network, we take a step from a point in the direction of that vector field (brown), which is analogous to the residual step in a ResNet.}
\label{fig:teaser2}
\vspace{-10px}
\end{figure*}

In this section, we cover the necessary background for our paper; in particular, we introduce the reader to the necessary constructs from Riemannian geometry. For a detailed introduction to Riemannian geometry, we refer the interested reader to textbooks such as \citet{Lee1997RiemannianMA}.

\subsection{Riemannian Geometry}
A topological manifold $(\M, g)$ of dimension $n$ is a locally Euclidean space, meaning there exist homeomorphic\footnote{A homeomorphism is a continuous bijection with continuous inverse.} functions (called ``charts") whose domains both cover the manifold and map from the manifold into $\mathbb{R}^n$ (i.e. the manifold ``looks like" $\mathbb{R}^n$ locally). A smooth manifold is a topological manifold for which the charts are not simply homeomorphic, but diffeomorphic, meaning they are smooth bijections mapping into $\mathbb{R}^n$ and have smooth inverses. We denote $T_p \M$ as the tangent space at a point $p$ of the manifold $\M$. Further still, a Riemannian manifold\footnote{Note that imposing Riemannian structure does not considerably limit the generality of our method, as any smooth manifold that is Hausdorff and second countable has a Riemannian metric \citep{Lee1997RiemannianMA}.} $(\M, g)$ is an $n$-dimensional smooth manifold with a smooth collection of inner products $(g_p)_{p \in \M}$ for every tangent space $T_p \M$. The Riemannian metric $g$ induces a distance $d_g : \M \times \M \rightarrow \mathbb{R}$ on the manifold.



\subsection{Geodesics and the Riemannian Exponential Map}

\textbf{Geodesics}
A geodesic is a curve of minimal length between two points $p,q \in \M$, and can be seen as the generalization of a straight line in Euclidean space. Although a choice of Riemannian metric $g$ on $\M$ appears to only define geometry locally on $\M$, it induces global distances by integrating the length (of the ``speed" vector in the tangent space) of a shortest path between two points:
\begin{equation}
d(p,q) = \inf_\gamma \int^1_0 \sqrt{g_{\gamma(t)} (\gamma'(t), \gamma'(t))}\:dt
\end{equation}

where $\gamma \in C^\infty ([0,1], \M)$ is such that $\gamma(0)=p$ and $\gamma(1) = q$.


For $p \in \M$ and $v \in T_p\M$, there exists a unique geodesic $\gamma_v$ where $\gamma(0) = p$, $\gamma'(0) = v$ and the domain of $\gamma$ is as large as possible. We call $\gamma_v$ the maximal geodesic \citep{Lee1997RiemannianMA}.


\textbf{Exponential Map}
The Riemannian exponential map is a way to map $T_p\M$ to a neighborhood around $p$ using geodesics. The relationship between the tangent space and the exponential map output can be thought of as a local linearization, meaning that we can perform typical Euclidean operations in the tangent space before projecting to the manifold via the exponential map to capture the local on-manifold behavior corresponding to the tangent space operations. For $p \in \M$ and $v \in T_p\M$, the exponential map at $p$ is defined as $\exp_p (v) = \gamma_v(1)$.

One can think of $\exp$ as a manifold generalization of Euclidean addition, since in the Euclidean case we have $\exp_p (v) = p + v$.

\vspace{-5pt}
\subsection{Vector Fields}
\vspace{-5pt}

Let $T_p \M$ be the tangent space to a manifold $\M$ at a point $p$. 
Like in Euclidean space, a vector field assigns to each point $p\in \M$ a tangent vector $X_p \in T_p \M$. A smooth vector field assigns a tangent vector $X_p\in T_p\M$ to each point $p\in \M$ such that $X_p$ varies smoothly in $p$.

\textbf{Tangent Bundle} 
The tangent bundle of a smooth manifold $\M$ is the disjoint union of the tangent spaces $T_p\M$, for all $p\in \M$, denoted by $T\M := \bigsqcup_{p \in \M} T_p \M = \bigsqcup_{p \in \M} \{(p,v) \mid v\in T_p\M\}$.




\textbf{Pushforward} A derivative (also called a \textit{pushforward}) of a map $f : \M \rightarrow \mathcal{N}$ between two manifolds is denoted by $D_p f : T_p \M \rightarrow T_{f(p)} \mathcal{N}$. This is a generalization of the classical Euclidean Jacobian (since $\mathbb{R}^n$ is a manifold), and provides a way to relate tangent spaces at different points on different manifolds.

\textbf{Pullback} Given $\phi : \M \rightarrow \mathcal{N}$ a smooth map between manifolds and $f : \mathcal{N} \rightarrow \R$ a smooth function, the pullback of $f$ by $\phi$ is the smooth function $\phi^* f$ on $\M$ defined by $(\phi^* f) (x) = f(\phi(x))$. When the map $\phi$ is implicit, we simply write $f^*$ to mean the pullback of $f$ by $\phi$.


\vspace{-5pt}
\subsection{Model Spaces in Riemannian Geometry}
\vspace{-5pt}

The three Riemannian model spaces are Euclidean space $\mathbb{R}^n$, hyperbolic space $\mathbb{H}^n$, and spherical space $\mathbb{S}^n$, that encompass all manifolds with constant sectional curvature. Hyperbolic space manifests in several representations like the Poincar\'e ball, Lorentz space, and the Klein model. We use the Poincar\'e ball model for our Riemannian ResNet design (see Appendix~\ref{sec:extended_background} for more details on the Poincar\'e ball model). 

\vspace{-5pt}
\subsection{SPD Manifold}
\vspace{-5pt}

Let $SPD(n)$ be the manifold of $n \times n$ symmetric positive definite (SPD) matrices.  We recall from \citet{Gallier2020DifferentialGA} that $SPD(n)$ has a Riemannian exponential map (at the identity) equivalent to the matrix exponential. Two common metrics used for $SPD(n)$ are the log-Euclidean metric \citep{Gallier2020DifferentialGA}, which induces a flat structure on the matrices, and the canonical affine-invariant metric \citep{Cruceru2021ComputationallyTR,Pennec2005ARF}, which induces non-constant negative sectional curvature. The latter gives $SPD(n)$ a considerably less trivial geometry than that exhibited by the Riemannian model spaces \citep{Bhatia2007PositiveDM} (see Appendix~\ref{sec:extended_background} for more details on $SPD(n)$).



\vspace{-5pt}
\section{Methodology}
\vspace{-5pt}
\begin{figure*}[t]
\centering
\includegraphics[width=\textwidth, page=5]{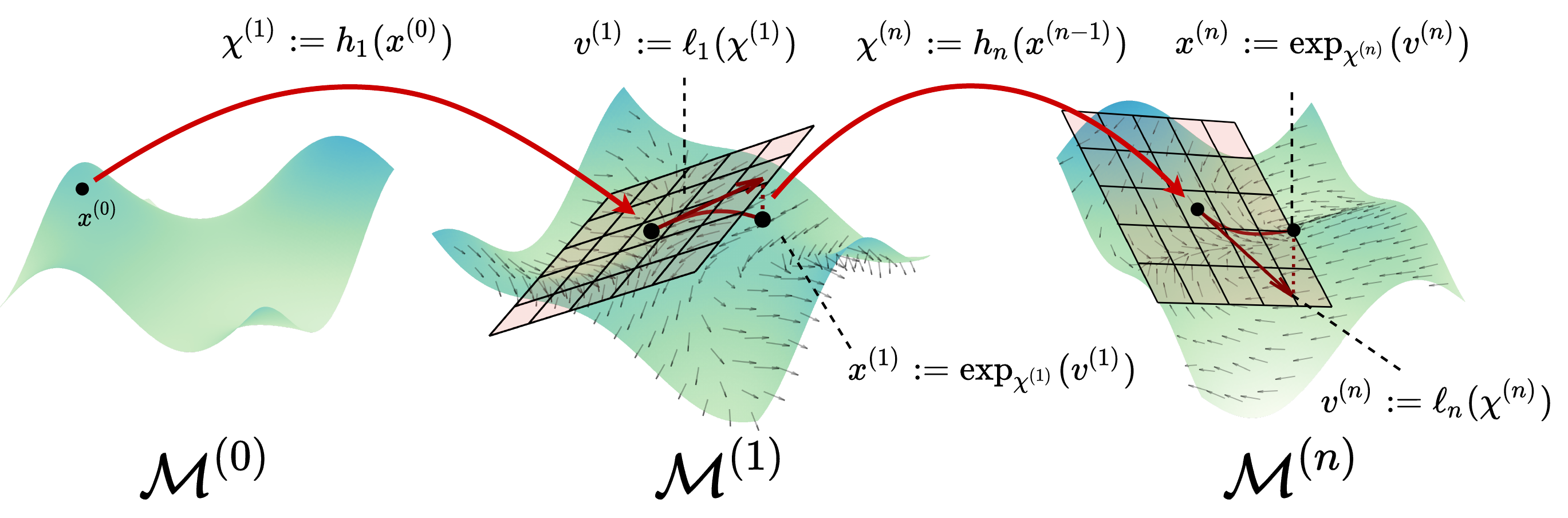}

\caption{An overview of our generalized Riemannian Residual Neural Network (RResNet) methodology. We start by mapping $x^{(0)}\in \mathcal{M}^{(0)}$ to  $\chi^{(1)} \in \mathcal{M}^{(1)}$ using a base point mapping $h_1$. Then, using our paramterized vector field $\ell_i$, we compute a residual $v^{(1)} := \ell_1(\chi^{(1)})$. Finally, we project $v^{(1)}$ back onto the manifold using the Riemannian $\exp$ map, leaving us with $x^{(1)}$. This procedure can be iterated to produce a multi-layer Riemannian residual neural network that is capable of changing manifold representation on a per layer basis.}
\label{fig:teaser}
\vspace{-10px}
\end{figure*}

\label{sec:methodology}

In this section, we provide the technical details behind Riemannian residual neural networks. 

\vspace{-5pt}
\subsection{General Construction}
\vspace{-5pt}

\label{sec:general}
We define a \textbf{Riemannian Residual Neural Network} (RResNet) on a manifold $\M$ to be a function $f: \M \to \M$ defined by
\vspace{-3px}
\begin{align}
    f(x) &:= x^{(m)}
    \\x^{(0)} &:= x 
    \\ x^{(i)} &:= \exp_{x^{(i - 1)}}(\ell_i(x^{(i - 1)})) 
\end{align}


for $x \in \M$, where $m$ is the number of layers and $\ell_i: \M \to T\M$ is a neural network-parameterized
vector field over $\M$. This residual network construction is visualized for the purpose of intuition in Figure \ref{fig:teaser2}. In practice, parameterizing a function from an abstract manifold $\M$ to its tangent bundle is difficult. However, by the Whitney embedding theorem \citep{lee2013introduction}, we can embed $\M \hookrightarrow \R^D$ smoothly for some dimension $D \ge \dim \M$. As such, for a standard neural network $n_i: \R^D \to \R^D$ we can construct $\ell_i$ by
\begin{equation}
    \ell_i(x) := {\rm proj}_{T_x\M}(n_i(x))
\end{equation}
where we note that $T_x\M \subset \R^D$ is a linear subspace (making the projection operator well defined). Throughout the paper we call this the embedded vector field design\footnote{Ideal vector field design is in general nontrivial and the embedded vector field is not a good choice for all manifolds (see Appendix \ref{sec:vecfielddesign}).}. We note that this is the same construction used for defining the vector field flow in \citet{lou2020neural, mathieu2020riemannian, Rozen2021MoserFD}.

We also extend our construction to work in settings where the underlying manifold changes from layer to layer. In particular, for a sequence of manifolds $\M^{(0)}, \M^{(1)}, \dots, \M^{(m)}$ with (possibly learned) maps $h_i: \M^{(i - 1)} \to \M^{(i)}$, our Riemannian ResNet $f : \M^{(0)} \to \M^{(m)}$ is given by
\begin{align}
    f(x) &:= x^{(m)}\\
    x^{(0)} &:= x\\
    x^{(i)} &:= \exp_{h_i(x^{(i - 1)})}(\ell_i(h_i(x^{(i - 1)})))  \forall i \in [m]
\end{align}


with functions $\ell_i: \M^{(i)} \to T\M^{(i)}$ given as above. This generalization is visualized in Figure \ref{fig:teaser}. In practice, our $\M^{(i)}$ will be different dimensional versions of the same geometric space (e.g. $\mathbb{H}^n$ or $\R^n$ for varying $n$). If the starting and ending manifolds are the same, the maps $h_i$ will simply be standard inclusions. When the starting and ending manifolds are different, the $h_i$ may be standard neural networks for which we project the output, or the $h_i$ may be specially design learnable maps that respect manifold geometry. As a concrete example, our $h_i$ for the SPD case map from an SPD matrix of one dimension to another by conjugating with a Stiefel matrix \cite{Huang2017ARN}. Furthermore, as shown in Appendix \ref{sec:theory}, our model is equivalent to the standard ResNet when the underlying manifold is $\mathbb{R}^n$.

\textbf{Comparison with Other Constructions} We discuss how our construction compares with other methods in Appendix \ref{sec:comparison_construction}, but here we briefly note that unlike other methods, our presented approach is fully general and better conforms with manifold geometry.



\vspace{-5pt}
\subsection{Feature Map-Induced Vector Field Design}
\vspace{-5pt}

\label{sec:featuremap}


Most of the difficulty in application of our general vector field construction comes from the design of the learnable vector fields $\ell_i : \M^{(i)} \rightarrow T\M^{(i)}$. Although we give an embedded vector field design above, it is not very principled geometrically. We would like to considerably restrict these vector fields so that their range is informed by the underlying geometry of $\M$. For this, we note that it is possible to induce a vector field $\xi : \M \rightarrow T\M$ for a manifold $\M$  with any smooth map $f : \M \rightarrow \R^k$. In practice, this map should capture intrinsic geometric properties of $\M$ and can be viewed as a feature map, or de facto linearization of $\M$. Given an $x \in \M$, we need only pass $x$ through $f$ to get its feature representation in $\mathbb{R}^k$, then note that since:
$$
D_p f : T_p \M \rightarrow T_{f(p)} \R^k,
$$
we have an induced map:
$$
(D_p f)^* : (T_{f(p)} \R^k)^* \rightarrow (T_p \M)^*,
$$
where $(D_pf)^*$ is the pullback of $D_p f$. Note that $T_{p} \mathbb{R}^k \cong \mathbb{R}^k$ and $(\mathbb{R}^k)^* \cong \mathbb{R}^k$ by the dual space isomorphism. Moreover $(T_p \M)^* \cong T_p \M$ by the tangent-cotangent space isomorphism \cite{lee2013introduction}. Hence, we have the induced map:
$$
(D_p f)^*_r : \mathbb{R}^k \rightarrow T_p \M,
$$
obtained from $(D_pf)^*$, simply by both precomposing and postcomposing the aforementioned isomorphisms, where relevant. $(D_p f)^*_r$ provides a natural way to map from the feature representation to the tangent bundle. Thus, we may view the map $\ell_f : \M \rightarrow T\M$ given by:
$$
\ell_f (x) = (D_x f)^*_r (f(x))
$$
as a deterministic vector field induced entirely by $f$.

\textbf{Learnable Feature Map-Induced Vector Fields} We can easily make the above vector field construction learnable by introducing a Euclidean neural network $n_\theta : \mathbb{R}^k \rightarrow \mathbb{R}^k$ after $f$ to obtain
$
\ell_{f, \theta} (x) = (D_x f)^* (n_\theta(f(x)))
$.

\textbf{Feature Map Design} One possible way to simplify the design of the above vector field is to further break down the map $f: \M \rightarrow \mathbb{R}^k$ into $k$ maps $f_1, \ldots, f_k : \M \rightarrow \R$, where ideally, each map $f_i$ is constructed in a similar way (e.g. performing some kind of geometric projection, where the $f_i$ vary only in terms of the specifying parameters). As we shall see in the following subsection, this ends up being a very natural design decision.

In what follows, we shall consider only smooth feature maps $f : \M \rightarrow \mathbb{R}^k$ induced by a single parametric construction $g_\theta : \M \rightarrow \R$, i.e. the $k$ dimensions of the output of $f$ are given by different choices of $\theta$ for the same underlying feature map\footnote{We use the term ``feature map" for both the overall feature map $f : \M \rightarrow \R^k$ and for the inducing construction $g_\theta : \M \rightarrow \R$. This is well-defined since in our work we consider only feature maps $f : \M \rightarrow \R^k$ that are induced by some $g_\theta : \M \rightarrow \R$.}. This approach also has the benefit of a very simple interpretation of the induced vector field. Given feature maps $g_{\theta_1}, \ldots, g_{\theta_k} : \M \rightarrow \R$ that comprise our overall feature map $f : \M \rightarrow \R^k$, our vector field is simply a linear combination of the maps $\nabla g_{\theta_i} : \M \rightarrow T\M$. If the $g_{\theta_i}$ are differentiable with respect to $\theta_i$, we can even learn the $\theta_i$ themselves.

\vspace{-3pt}
\subsubsection*{4.2.1\:\:\:Manifold Manifestations}
\vspace{-3pt}

In this section, in an effort to showcase how simple it is to apply our above theory to come up with natural vector field designs, we present several constructions of manifold feature maps $g_\theta : \M \rightarrow \R$ that capture the underlying geometry of $\M$ for various choices of $\M$. Namely, in this section we provide several examples of $f : \M \rightarrow \R$ that induce $\ell_f : \M \rightarrow T\M$, thereby giving rise to a Riemannian neural network by Section \ref{sec:general}.


\begin{wrapfigure}{U}{.4\textwidth} 
\centering
\includegraphics[width=0.3\textwidth]{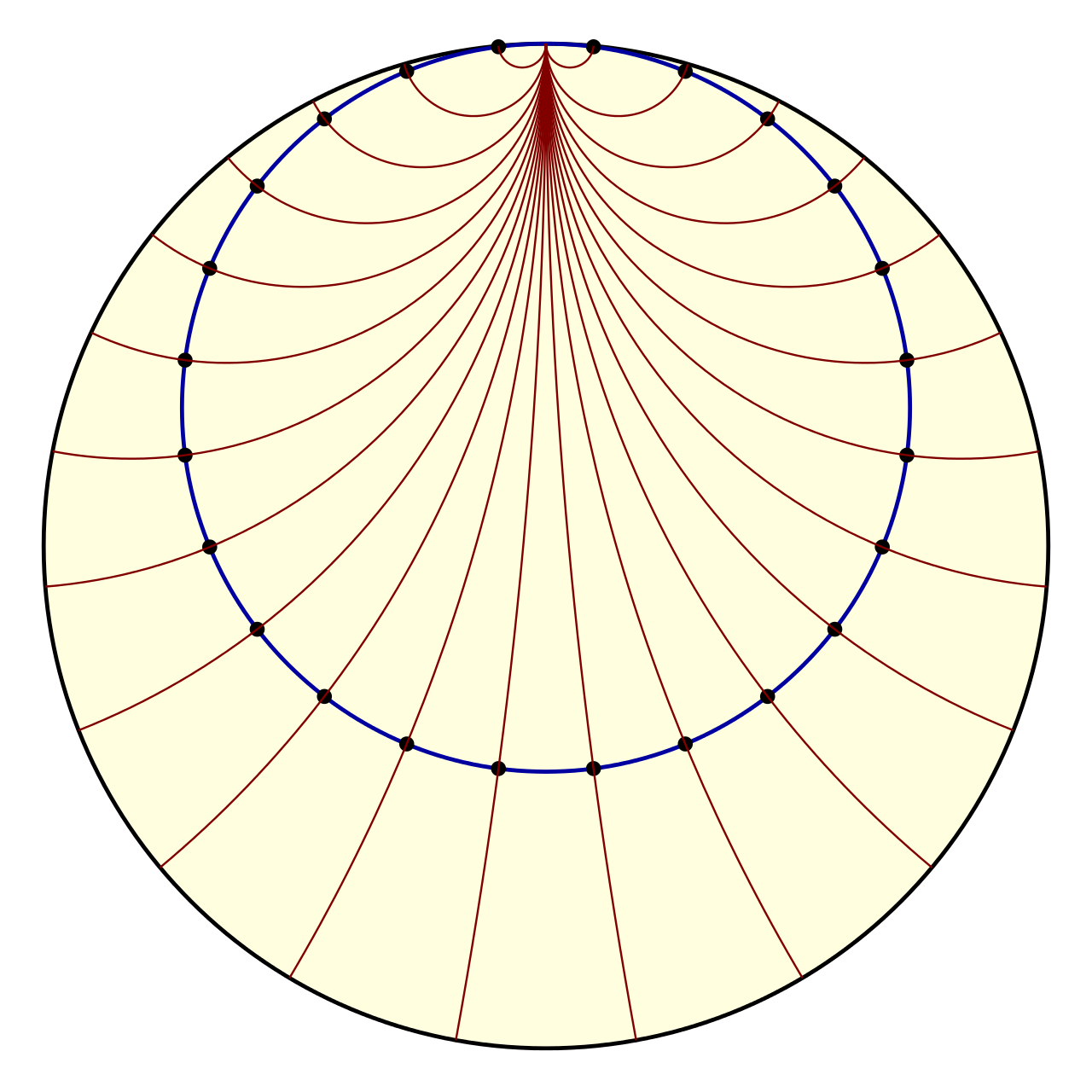}
\caption{\label{fig:horosphere}Example of a horosphere in the Poincar\'e ball representation of hyperbolic space. In this particular two-dimensional case, the hyperbolic space $\mathbb{H}_2$ is visualized via the Poincar\'e disk model, and the horosphere, shown in blue, is called a horocycle.}
\vspace{-10pt}
\end{wrapfigure}

\textbf{Euclidean Space} To build intuition, we begin with an instructive case. We consider designing a feature map for the Euclidean space $\mathbb{R}^n$. A natural design would follow simply by considering hyperplane projection. Let a hyperplane $w^T x + b = 0$ be specified by $w \in \mathbb{R}^n, b \in \mathbb{R}$. Then a natural feature map $g_{w,b} : \mathbb{R}^n \rightarrow \mathbb{R}$ parameterized by the hyperplane parameters is given by hyperplane projection \cite{hyperplanes}:
$g_{w,b}(x) = \frac{|w^T x + b|}{||w||_2}$.


\textbf{Hyperbolic Space} We wish to construct a natural feature map for hyperbolic space. Seeking to follow the construction given in the Euclidean context, we wish to find a hyperbolic analog of hyperplanes. This is provided to us via the notion of horospheres \cite{Heintze1977GeometryOH}. Illustrated in Figure \ref{fig:horosphere}, horospheres naturally generalize hyperplanes to hyperbolic space. We specify a horosphere in the Poincar\'e ball model of hyperbolic space $\mathbb{H}^n$ by a point of tangency $\omega \in \mathbb{S}^{n-1}$ and a real value $b \in \mathbb{R}$. Then a natural feature map $g_{\omega, b} : \mathbb{H}^n \rightarrow \mathbb{R}$ parameterized by the horosphere parameters would be given by horosphere projection \citep{Bridson1999MetricSO}: $g_{\omega, b} (x) = -\log\left(\frac{1-||x||^2_2}{||x-\omega||^2_2} \right) + b$.



\textbf{Symmetric Positive Definite Matrices} The manifold of SPD matrices is an example of a manifold where there is no innate representation of a hyperplane. Instead, given $X \in SPD(n)$, a reasonable feature map $g_k: SPD(n) \to \R$, parameterized by $k$, is to map $X$ to its $k$th largest eigenvalue: $g_k(X) = \lambda_k$. 

\textbf{General Manifolds} For general manifolds there is no perfect analog of a hyperplane, and hence there is no immediately natural feature map. Although this is the 
case, it is possible to come up with a reasonable alternative. We present such an alternative in Appendix \ref{sec:featuremapgeneralman} together with pertinent experiments.

\textbf{\textit{Example: Euclidean Space}} One motivation for the vector field construction $\ell_f (x) = (D_x f)^*_r (f(x))$ is that in the Euclidean case, $\ell_f$ will reduce to a standard linear layer (because the maps $f$ and $(D_x f)^*$ are linear), which, in combination with the Euclidean $\exp$ map, will produce a standard Euclidean residual neural network.

Explicitly, for the Euclidean case, note that our feature map $f : \mathbb{R}^n \rightarrow \mathbb{R}^k$ will, for example, take the form $f(x) = Wx, W \in \mathbb{R}^{k \times n}$ (here we have $b=0$ and $W$ has normalized row vectors). Then note that we have $Df = W$ and $(Df)^* = W^T$. We see for the standard feature map-based construction, our vector field $\ell_f (x) = (D_x f)^* (f(x))$ takes the form $\ell_f (x) = W^T W x$.

For the learnable case (which is standard for us, given that we learn Riemannian residual neural networks), when the manifold is Euclidean space, the general expression $\ell_{f, \theta} (x) = (D_x f)^* (n_\theta(f(x)))$ becomes $\ell_{f,\theta} (x) = W^T n_\theta (W x)$. When the feature maps are trivial projections (onto axis-aligned hyperplanes), we have $W= I$ and $\ell_{f,\theta} (x) = n_\theta(x)$. Thus our construction can be viewed as a generalization of a standard neural network.





\begin{table*}[htb!]
\centering
\scalebox{0.80}{
\begin{tabular}{clcccccccc}
\toprule
& \textbf{Dataset} & \multicolumn{2}{c}{Disease} & \multicolumn{2}{c}{Airport} & \multicolumn{2}{c}{PubMed} & \multicolumn{2}{c}{CoRA} \\
& \textbf{Hyperbolicity} & \multicolumn{2}{c}{$\delta = 0$} & \multicolumn{2}{c}{$\delta = 1$} & \multicolumn{2}{c}{$\delta = 3.5$} & \multicolumn{2}{c}{$\delta = 11$} \\ \cmidrule{2-10}
& \textbf{Task} & LP & NC & LP & NC & LP & NC & LP & NC  \\
\midrule
\parbox[t]{2mm}{\multirow{4}{*}{\rotatebox[origin=c]{90}{\scriptsize{Shallow}}}} & Euc & $59.8$\tiny$\pm 2.0$ & $32.5$\tiny$\pm 1.1$ & $92.0$\tiny$\pm 0.0$ & $60.9$\tiny$\pm 3.4$ & $83.3$\tiny$\pm 0.1$ & $48.2$\tiny$\pm 0.7$ & $82.5$\tiny$\pm 0.3$ & $23.8$\tiny$\pm 0.7$ \\
& Hyp \cite{nickel2017poincare} & $63.5$\tiny$\pm 0.6$ & $45.5$\tiny$\pm 3.3$ & $94.5$\tiny$\pm 0.0$ & $70.2$\tiny$\pm 0.1$ & $87.5$\tiny$\pm 0.1$ & $68.5$\tiny$\pm 0.3$ & $87.6$\tiny$\pm 0.2$ & $22.0$\tiny$\pm 1.5$ \\
& Euc-Mixed & $49.6$\tiny$\pm 1.1$ & $35.2$\tiny$\pm 3.4$ & $91.5$\tiny$\pm 0.1$ & $68.3$\tiny$\pm 2.3$ & $86.0$\tiny$\pm 1.3$ & $63.0$\tiny$\pm 0.3$ & $84.4$\tiny$\pm 0.2$ & $46.1$\tiny$\pm 0.4$ \\
& Hyp-Mixed & $55.1$\tiny$\pm 1.3$ & $56.9$\tiny$\pm 1.5$ & $93.3$\tiny$\pm 0.0$ & $69.6$\tiny$\pm 0.1$ & $83.8$\tiny$\pm 0.3$ & $\mathbf{73.9}$\tiny$\pm 0.2$ & $85.6$\tiny$\pm 0.5$ & $45.9$\tiny$\pm 0.3$ \\
\midrule
\parbox[t]{2mm}{\multirow{3}{*}{\rotatebox[origin=c]{90}{\scriptsize{NN}}}}
& MLP & $72.6$\tiny$\pm 0.6$ & $28.8$\tiny$\pm 2.5$ & $89.8$\tiny$\pm 0.5$ & $68.6$\tiny$\pm 0.6$ & $84.1$\tiny$\pm 0.9$ & $72.4$\tiny$\pm 0.2$ & $83.1$\tiny$\pm 0.5$ & $51.5$\tiny$\pm 1.0$ \\
& HNN \cite{ganea2018hyperbolic} & $75.1$\tiny$\pm 0.3$ & $41.0$\tiny$\pm 1.8$ & $90.8$\tiny$\pm 0.2$ & $80.5$\tiny$\pm 0.5$ & $\mathbf{94.9}$\tiny$\pm 0.1$ & $69.8$\tiny$\pm 0.4$ & $\mathbf{89.0}$\tiny$\pm 0.1$ & $\mathbf{54.6}$\tiny$\pm 0.4$ \\
& \textbf{RResNet Horo} & $\mathbf{98.4}$\tiny$\pm 0.3$ & $\mathbf{76.8}$\tiny$\pm 2.0$ & $\mathbf{95.2}$\tiny$\pm 0.1$ & $\mathbf{96.9}$\tiny$\pm 0.3$ & $\mathbf{95.0}$\tiny$\pm 0.3$ & $72.3$\tiny$\pm 1.7$ & $86.7$\tiny$\pm 6.3$ & $52.4$\tiny$\pm 5.5$ \\ 
\bottomrule
\end{tabular}}
\caption{Above we give graph task results for RResNet Horo compared with several non-graph-based neural network baselines (baseline methods and metrics are from \citet{chami2019hyperbolic}). Test ROC AUC is the metric reported for link prediction (LP) and test F1 score is the metric reported for node classification (NC). Mean and standard deviation are given over five trials. Note that RResNet Horo considerably outperforms HNN on the most hyperbolic datasets, performing worse and worse as hyperbolicity increases, to a more extreme extent than previous methods that do not adhere to geometry as closely (this is expected).}
\label{tab:hyperbolicmain}
\end{table*}

\vspace{-5pt}
\section{Experiments}
\vspace{-5pt}

\label{sec:experiments}

In this section, we perform a series of experiments to evaluate the effectiveness of RResNets on tasks arising on different manifolds. In particular, we explore hyperbolic space and the SPD manifold.

\vspace{-5pt}
\subsection{Hyperbolic Space}
\vspace{-5pt}
\label{sec:hypinit}

We perform numerous experiments in the hyperbolic setting. The purpose is twofold:

\begin{enumerate}[itemsep=0.5mm, leftmargin=20px]
    \item We wish to illustrate that our construction in Section \ref{sec:methodology} is not only more general, but also intrinsically more geometrically natural than pre-existing hyperbolic constructions such as HNN \cite{ganea2018hyperbolic}, and is thus able to learn better over hyperbolic data.
    \item We would like to highlight that non-Euclidean learning benefits the most hyperbolic datasets. We can do this directly since our method provides a way to vary the geometry of a fixed neural network architecture, thereby allowing us to directly investigate the effect of changing geometry from Euclidean to hyperbolic. 
\end{enumerate}

\subsubsection*{5.1.1\:\:\:\:Direct Comparison Against Hyperbolic Neural Networks \cite{ganea2018hyperbolic}}

To demonstrate the improvement of RResNet over HNN \cite{ganea2018hyperbolic}, we first perform node classification (NC) and link prediction (LP) tasks on graph datasets with low Gromov $\delta$-hyperbolicity \citep{chami2019hyperbolic}, which means the underlying structure of the data is highly hyperbolic. The RResNet model is given the name ``RResNet Horo." It utilizes a horosphere projection feature map-induced vector field described in Section \ref{sec:methodology}. All model details are given in Appendix \ref{sec:hyperbolicarch&traindetails}. We find that because we adhere well to the geometry, we attain good performance on datasets with low Gromov $\delta$-hyperbolicities (e.g. $\delta = 0, \delta = 1$). As soon as the Gromov hyperbolicity increases considerably beyond that (e.g. $\delta = 3.5, \delta = 11$), performance begins to degrade since we are embedding non-hyperbolic data in an unnatural manifold geometry. Since we adhere to the manifold geometry more strongly than prior hyperbolic work, we see performance decay faster as Gromov hyperbolicity increases, as expected. In particular, we test on the very hyperbolic Disease ($\delta = 0$) \citep{chami2019hyperbolic} and Airport ($\delta = 1$) \cite{chami2019hyperbolic} datasets. We also test on the considerably less hyperbolic PubMed ($\delta = 3.5$) \cite{Pubmed} and CoRA ($\delta = 11$) \cite{Cora} datasets. We use all of the non-graph-based baselines from \citet{chami2019hyperbolic}, since we wish to see how much we can learn strictly from a proper treatment of the embeddings (and no graph information). Table~\ref{tab:hyperbolicmain} summarizes the performance of ``RResNet Horo"  relative to these baselines.



Moreover, we find considerable benefit from the feature map-induced vector field over an embedded vector field that simply uses a Euclidean network to map from a manifold point embedded in $\mathbb{R}^n$. The horosphere projection captures geometry more accurately, and if we swap to an embedded vector field we see considerable accuracy drops on the two hardest hyperbolic tasks: Disease NC and Airport NC. In particular, for Disease NC the mean drops from $76.8$ to $75.0$, and for Airport NC we see a very large decrease from $96.9$ to $83.0$, indicating that geometry captured with a well-designed feature map is especially important. We conduct a more thorough vector field ablation study in Appendix \ref{sec:hypablation}.

\subsubsection*{5.1.2\:\:\:\:Impact of Geometry}


A major strength of our method is that it allows one to investigate the direct effect of geometry in obtaining results, since the architecture can remain the same for various manifolds and geometries (as specified by the metric of a given Riemannian manifold). This is well-illustrated in the most hyperbolic Disease NC setting, where swapping out hyperbolic for Euclidean geometry in an RResNet induced by an embedded vector field decreases the F1 score from a $75.0$ mean to a $67.3$ mean and induces a large amount of numerical stability, since standard deviation increases from $5.0$ to $21.0$. We conduct a more thorough geometry ablation study in Appendix \ref{sec:hypablation}. 

\begin{table*}
\vspace{-10px}
\centering
\scalebox{0.9}{
\begin{tabular}{lcccc}
\toprule
& AFEW\citep{Dhall2011StaticFE} & FPHA\citep{GarciaHernando2018FirstPersonHA} & NTU RGB+D\citep{shahroudy2016ntu} & HDM05\citep{HDM05}\\ 
\midrule
SPDNet & $33.24$\tiny$\pm 0.56$& $65.39$\tiny$ \pm 1.48$ & $41.47$\tiny$ \pm 0.34$ &  $66.77 $\tiny$\pm 0.92$\\
SPDNetBN & $35.39 $\tiny$\pm 0.93$& $65.03 $\tiny$\pm 1.35$ & $41.92$\tiny$\pm 0.37$ & $67.25 $\tiny$\pm 0.44$\\
\textbf{RResNet Affine-Invariant} & $35.17$ \tiny $\pm 1.78$&$\mathbf{66.53} $\tiny$\mathbf{\pm 01.64}$ & $41.00$ \tiny $\pm 0.50$& $67.91$ \tiny${\pm 1.27}$\\
\textbf{RResNet Log-Euclidean}&$\mathbf{36.38} $\tiny$\mathbf{\pm 1.29}$& $64.58 $\tiny$\pm 0.98$& $\mathbf{42.99 }$\tiny$\mathbf{\pm 0.23}$& $\mathbf{69.80}$\tiny$\mathbf{ \pm 1.51}$ \\
\bottomrule
\end{tabular}}
\caption{We run our SPD manifold RResNet on four SPD matrix datasets and compare against SPDNet \citep{Huang2017ARN} and SPDNet with batch norm \citep{Brooks2019RiemannianBN}. We report the mean and standard deviation of validation accuracies over five trials and bold which method performs the best.}
\label{tab:spd_table}
\vspace{-15px}
\end{table*}

\subsection{SPD Manifold}
A common application of SPD manifold-based models is learning over full-rank covariance matrices, which lie on the manifold of SPD matrices. We compare our  RResNet  to SPDNet \citep{Huang2017ARN} and SPDNet with batch norm \citep{Brooks2019RiemannianBN} on four video classification datasets: AFEW \citep{Dhall2011StaticFE}, FPHA \citep{GarciaHernando2018FirstPersonHA}, NTU RGB+D \citep{shahroudy2016ntu}, and HDM05 \citep{HDM05}.
Results are given in Table \ref{tab:spd_table}. Please see Appendix \ref{sec:spd_appendix} for details on the experimental setup. For our RResNet design, we try two different metrics: the log-Euclidean metric \citep{Gallier2020DifferentialGA} and the affine-invariant metric \citep{Cruceru2021ComputationallyTR,Pennec2005ARF}, each of which captures the curvature of the SPD manifold differently. We find that adding a learned residual improves performance and training dynamics over existing neural networks on SPD manifolds with little effect on runtime. We experiment with several vector field designs, which we outline in Appendix \ref{sec:vecfielddesign}. The best vector field design (given in Section \ref{sec:featuremap}), also the one we use for all SPD experiments, necessitates eigenvalue computation. We note the cost of computing eigenvalues is not a detrimental feature of our approach since previous works (SPDNet \cite{Huang2017ARN}, SPDNet with batchnorm \cite{Brooks2019RiemannianBN}) already make use of eigenvalue computation\footnote{One needs this computation for operations such as the Riemannian $\exp$ and $\log$ over the SPD manifold.}. Empirically, we observe that the beneficial effects of our RResNet construction are similar to those of the SPD batch norm introduced in \citet{Brooks2019RiemannianBN} (Table \ref{tab:spd_table}, Figure \ref{fig:spd_training} in Appendix \ref{sec:spd_appendix}). In addition, we find that our operations are stable with ill-conditioned input matrices, which commonly occur in the wild. To contrast, the batch norm computation in SPDNetBN, which relies on Karcher flow \citep{Karcher77,lou2020differentiating}, suffers from numerical instability when the input matrices are nearly singular. Overall, we observe our RResNet with the affine-invariant metric outperforms existing work on FPHA, and our RResNet using the log-Euclidean metric outperforms existing work on AFEW, NTU RGB+D, and HDM05. Being able to directly interchange between two metrics while maintaining the same neural network design is an unique strength of our model.

\begin{table*}
\centering
\scalebox{0.9}{\begin{tabular}{clcccc}
\toprule
& \textbf{Dataset} & Disease & Airport & PubMed & CoRA \\
& \textbf{Hyperbolicity} & $\delta = 0$ & $\delta = 1$ & $\delta = 3.5$ & $\delta = 11$ \\
\midrule
\parbox[t]{2mm}{\multirow{4}{*}{\rotatebox[origin=c]{90}{\scriptsize{GNN}}}} & GCN \cite{Kipf2017SemiSupervisedCW} & $69.7$\tiny$\pm 0.4$ & $81.4$\tiny$\pm 0.6$ & $78.1$\tiny$\pm 0.2$ & $81.3$\tiny$\pm 0.3$ \\
& GAT \cite{Velickovic2018GraphAN} & $70.4$\tiny$\pm 0.4$ & $81.5$\tiny$\pm 0.3$ & $79.0$\tiny$\pm 0.3$ & $\mathbf{83.0}$\tiny$\pm 0.7$ \\
& SAGE \cite{Hamilton2017InductiveRL} & $69.1$\tiny$\pm 0.6$ & $82.1$\tiny$\pm 0.5$ & $77.4$\tiny$\pm 2.2$ & $77.9$\tiny$\pm 2.4$ \\
& SGC \cite{Wu2019SimplifyingGC} & $69.5$\tiny$\pm 0.2$ & $80.6$\tiny$\pm 0.1$ & $78.9$\tiny$\pm 0.0$ & $81.0$\tiny$\pm 0.1$ \\
\midrule
\parbox[t]{2mm}{\multirow{3}{*}{\rotatebox[origin=c]{90}{\scriptsize{GGNN}}}} & HGCN \cite{chami2019hyperbolic} & $74.5$\tiny$\pm 0.9$ & $90.6$\tiny$\pm 0.2$ & $\mathbf{80.3}$\tiny$\pm 0.3$ & $79.9$\tiny$\pm 0.2$ \\
& Fully HNN \cite{chen-etal-2022-fully} & $\mathbf{96.0}$\tiny$\pm 1.0$ & $90.9$\tiny$\pm 1.4$ & $78.0$\tiny$\pm 1.0$ & $80.2$\tiny$\pm 1.3$ \\
& \textbf{G-RResNet Horo} & $\mathbf{95.4}$\tiny$\pm 1.0$ & $\mathbf{97.4}$\tiny$\pm 0.1$ & $75.5$\tiny$\pm 0.8$ & $64.4$\tiny$\pm 7.6$ \\
\bottomrule
\end{tabular}}
\caption{Above we give node classification results for G-RResNet Horo compared with several graph-based neural network baselines (baseline methods and metrics are from \citet{chami2019hyperbolic}). Test F1 score is the metric reported. Mean and standard deviation are given over five trials. Note that G-RResNet Horo obtains a state-of-the-art result on Airport. As for the less hyperbolic datasets, G-RResNet Horo does worse on PubMed and does very poorly on CoRA, once more, as expected due to unsuitability of geometry. The GNN label stands for ``Graph Neural Networks" and the GGNN label stands for ``Geometric Graph Neural Networks."}
\label{tab:hyperbolicgraph}
\vspace{-10pt}
\end{table*}

\vspace{-10pt}
\section{Riemannian Residual Graph Neural Networks}
\vspace{-7pt}


Following the initial comparison to non-graph-based methods in Table \ref{tab:hyperbolicmain}, we introduce a simple graph-based method by modifying RResNet Horo above. We take the previous model and pre-multiply the feature map output by the underlying graph adjacency matrix $A$ in a manner akin to what happens with graph neural networks \cite{Wu2019SimplifyingGC}. This is the simple modification that we introduce to the Riemannian ResNet to incorporate graph information; we call this method G-RResNet Horo. We compare directly against the graph-based methods in \citet{chami2019hyperbolic} as well as against Fully Hyperbolic Neural Networks \cite{chen-etal-2022-fully} and give results in Table~\ref{tab:hyperbolicgraph}. We test primarily on node classification since we found that almost all LP tasks are too simple and solved by methods in \citet{chami2019hyperbolic} (i.e., test ROC is greater than $95\%$). We also tune the matrix power of $A$ for a given dataset; full architectural details are given in Appendix \ref{sec:hyperbolicarch&traindetails}. Although this method is simple, we see further improvement and in fact attain a state-of-the-art result for the Airport \cite{chami2019hyperbolic} dataset. Once more, as expected, we see a considerable performance drop for the much less hyperbolic datasets, PubMed and CoRA.

\vspace{-10pt}
\section{Conclusion}
\vspace{-8pt}

We propose a general construction of residual neural networks on Riemannian manifolds. Our approach is a natural geodesically-oriented generalization that can be applied more broadly than previous manifold-specific work. Our introduced neural network construction is the first that decouples geometry (i.e. the representation space expected for input to layers) from the architecture design (i.e. actual “wiring” of the layers).
Moreover, we introduce a geometrically principled feature map-induced vector field design for the RResNet. We demonstrate that our methodology better captures underlying geometry than existing manifold-specific neural network constructions. On a variety of tasks such as node classification, link prediction, and covariance matrix classification, our method outperforms previous work. Finally, our RResNet's principled construction allows us to directly assess the effect of geometry on a task, with neural network architecture held constant. We illustrate this by directly comparing the performance of two Riemannian metrics on the manifold of SPD matrices.
We hope others will use our work to better learn over data with nontrivial geometries in relevant fields, such as lattice quantum field theory, robotics, and computational chemistry. 

\textbf{Limitations} We rely fundamentally on knowledge of geodesics of the underlying manifold. As such, we assume that a closed form (or more generally, easily computable, differentiable form) is given for the Riemannian exponential map as well as for the tangent spaces.

\section*{Acknowledgements}

We would like to thank Facebook AI for funding equipment that made this work possible. In addition, we thank the National Science Foundation for awarding Prof. Christopher De Sa a grant that helps fund this research effort (NSF IIS-2008102) and for supporting both Isay Katsman and Aaron Lou with graduate research fellowships. We would also like to acknowledge Prof. David Bindel for his useful insights on the numerics of SPD matrices.

\bibliography{rresnet}
\bibliographystyle{plainnat}





\appendix

\newpage
{\LARGE\textbf{Appendix}}

\appendix
\section{Riemannian Geometry: Relevant Reference Material}
\label{sec:extended_background}

Here we give some relevant reference material that provides the reader with the fundamental operations used for the Poincar\'e ball model of hyperbolic space, as well as the two Riemannian SPD manifold structures employed.

\subsection{The Poincar\'e Ball Model}

Hyperbolic space can be represented via several isometric models. We use the Poincar\'e ball model, which is defined by the set 
\begin{equation}
    \Big\{ x \in \mathbb{R}^n \mid \|x\|_2^2 < -\frac{1}{K}\Big\},
\end{equation}
where $K < 0$ is the space's constant negative curvature together with the metric given in the table below. We give a summary of hyperbolic operations in Table~\ref{tab:hyp_operations}.

\begin{table*}[h]
\centering
\renewcommand{\arraystretch}{1.5}
\begin{tabular}{c|ccc}
 \toprule
\textbf{Manifold} & \textbf{Euclidean} $\mathbb{R}^n$ &  \textbf{ Poincar\'e Ball} $\mathbb{H}^n$\\
 \midrule
 Dimension, $\text{dim}(\M)$ &  $n$  & $n$ \\
 Metric $g_x$,& $g^{\mathbb{E}}$ & $(\lambda_x^K)g^{\mathbb{E}}$, where $g^{\mathbb{E}} = I$ \\
 Tangent Space, $T_x\M$ & $\mathbb{R}^n$ & $\mathbb{R}^n$\\
  Projection, $\text{proj}_{T_x\M}(v)$&  $v$ & $v$  \\
Exp Map, $\exp_x(v)$ & $ x + v$ & $ x \oplus_K \left(\tanh \left(\sqrt{|K|}\frac{\lambda_x^K\|v\|_2}{2}\right)\frac{v}{\sqrt{|K|}\|v\|_2}\right)$\\
Geodesic Distance, $d(x, y)$& $\|y-x\|_2$ &  $\frac{1}{\sqrt{|K|}} \cosh^{-1}\left(1-\frac{2K\|x-y\|_2^2}{(1 + K\|x\|_2^2)(1 + K\|y\|_2^2})\right)$  \\
\bottomrule

\end{tabular}
\caption{Summary of Poincar\'e ball operations. We provide equivalent operations on Euclidean space for reference. $\oplus_K$ denotes M\"{o}bius addition~\citep{Ungar2009AGS}, and $\lambda_x^K = \frac{2}{1 + K\|x\|_2^2}$, a conformal factor. }
\label{tab:hyp_operations}
\end{table*}

\subsection{The SPD Manifold}
We provide a summary of operations on the manifold of SPD matrices, in Table \ref{tab:operations}. For the SPD manifold, we illustrate the differences between the affine-invariant and log-Euclidean metrics. $\exp$ and $\log$ denote the matrix exponential and logarithm, respectively. 
\begin{table*}[h]
\centering
\renewcommand{\arraystretch}{1.5}
\scalebox{0.9}{\begin{tabular}{c|ccc}
 \toprule
\textbf{Manifold} & \textbf{Euclidean} $\mathbb{R}^n$ &  $\textbf{SPD}(n)$ Affine-Invariant &$\textbf{SPD}(n)$ Log-Euclidean  \\
 \midrule
 Dimension, $\text{dim}(\M)$ &  $n$  & $\frac{n(n+1)}{2}$ & $\frac{n(n+1)}{2}$\\
 Metric $g_x$, & $g^\mathbb{E}$ & $\text{tr}(X^{-1}UX^{-1}V)$ & $\text{tr}((D\log_X(U))^TD\log_X(V))$ \\
 Tangent Space, $T_x\M$ & $\mathbb{R}^n$ &  $\{V \mid V = V^T\}$ & $ \{V \mid V = V^T\}$\\
  Projection, $\text{proj}_{T_x\M}(v)$&  $v$  & $\frac{V + V^T}{2}$ &  $ \frac{V + V^T}{2}$\\
Exp Map, $\exp_x(v)$ & $ x + v$ & $ X\exp(X^{-1}V)$& $\exp(\log(X) + V)$\\
Geodesic Distance, $d(x, y)$& $ \|y - x\|_2$  & $ \|\log(X^{-1}Y)\|_F$& $\|\log(Y) - \log(X)\|_F$\\
\bottomrule

\end{tabular}}
\caption{Summary of SPD operations. We provide equivalent operations on Euclidean space for reference. We use both the affine-invariant and log-Euclidean metrics. }
\label{tab:operations}
\end{table*}

\section{Vector Field Design}
\label{sec:vecfielddesign}

Recall from the main paper that we can design a neural network-parameterized vector field $\ell_i : \M \rightarrow T\M$ for an embedded manifold $\M$ of dimension $D$, simply by defining a standard neural network $n_i : \mathbb{R}^D \rightarrow \mathbb{R}^D$ and then setting:
\begin{equation}
\ell_i(x) := \text{proj}_{T_x \M} (n_i(x)).
\end{equation}

Though this vector field design is frequently trivial (assuming the manifold has a natural embedding in $\mathbb{R}^n$), it may be highly inefficient if an easy-to-implement but suboptimal embedding is used. This is especially the case if manifold structure is underexploited in the construction of such an embedding (see Section \ref{sec:general}). In this section, we give a natural embedded vector field design for hyperbolic space, a more geometric feature map-induced vector field design for hyperbolic space, and explore a variety of possible vector field designs for the SPD manifold. In the general setting, note that obtaining a parsimonious (with respect to either representational dimension or parameter count) vector field design that is sufficiently expressive is nontrivial.

\subsection{Vector Field Design for Hyperbolic Space}

For the embedded hyperbolic vector field design, we apply the general design construction referenced above. Note that $\mathbb{H}^n$ is an $n$-dimensional manifold with a trivial $\mathbb{R}^{n+1}$ embedding given by any coordinate representation. Thus we need only parameterize a neural network $n_i : \mathbb{R}^{n+1} \rightarrow \mathbb{R}^{n+1}$ and set
\begin{equation}
\ell_i(x) = \text{proj}_{T_x \mathbb{H}^n} (n_i (x))
\end{equation}

to obtain our neural network-parameterized vector fields. Observe that this vector field design is efficient and expressive, since $T_x \mathbb{H}^n \cong \mathbb{R}^{n}$, but is perhaps too expressive in that the vector field is not constructed around the geodesic geometry of hyperbolic space. For this, we employ the horosphere projection-induced vector field design introduced in Section \ref{sec:featuremap} of the main paper. We simply fix a number of horospheres, randomly initialize them, and then further learn hyperparameters specifying a given horosphere. 

\subsection{Vector Field Design for the SPD Manifold}
\label{sec:a2}

Let $SPD(n)$ be the manifold of $n \times n$ SPD matrices with canonical metric, as in the main paper. We recall from \citet{Gallier2020DifferentialGA} that $SPD$ has a Lie structure with algebra consisting of $n\times n$ symmetric matrices, denoted $S(n)$. The Riemannian exponential map (or equivalently, the matrix exponential map) is a bijection between $S(n)$ and $SPD(n)$. Recall by Lie symmetry \citep{Gallier2020DifferentialGA} that the tangent space at $X \in SPD(n)$ is given by:
\begin{equation}
T_X SPD(n) = S(n) := \{P \mid P  = P^T\}.
\end{equation}

Observe that due to this tangent space structure, instead of utilizing the vector field construction given in Section \ref{sec:general} that requires an explicit projection operator, we may opt for more amenable designs oriented around the SPD manifold's Lie structure. We develop a variety of constructions below.

\textbf{B.2.1 \:\: Design 1: Embedded}

We can observe that $SPD(n)$ is trivially embedded in $\mathbb{R}^{n^2}$, and so are its tangent vectors; we will use this observation to construct a simple vector field parameterization. Let $\text{vec}: \mathbb{R}^{n \times n} \rightarrow \mathbb{R}^{n^2}$ be row-major matrix vectorization and let $\text{vec}^{-1} : \mathbb{R}^{n^2} \rightarrow \mathbb{R}^{n \times n}$ be its inverse. Given a neural network $n_i : \mathbb{R}^{n^2} \rightarrow \mathbb{R}^{n^2}$ and an $X \in SPD(n)$, we may set:
\begin{equation}
\ell_i(X) = \text{proj}_{T_XSPD(n)}(\text{vec}^{-1} (n_i(\text{vec}(X))))
\end{equation}

where $\text{proj}_{T_XSPD(n)}$ symmetrizes a matrix in the tangent space of the identity matrix, before transforming it back to the tangent space of $X$. It is given by:

\begin{equation}
\text{proj}_{T_XSPD(n)} (V) = \frac{V + V^T}{2}.
\end{equation}

Although this vector field representation is expressive, it also provides unneeded flexibility. For example, the intrinsic dimension of $T_X SPD(n) \cong S(n)$ is $\frac{n(n+1)}{2}$, but the $n_i$ map to all of $\mathbb{R}^{n^2}$. Based on this observation, we exploit tangent vector structure in the following vector field design to retain expressiveness while increasing efficiency.

\textbf{B.2.2 \:\: Design 2: Structured}

Observe that our tangent spaces satisfy $T_X SPD(n) \cong S(n)$, and moreover that $SPD(n) \subset S(n)$. We know that $S(n)$ has dimension $\frac{n(n+1)}{2}$ since each symmetric matrix is uniquely determined by its upper triangular part. Let $\iota : \mathbb{R}^{\frac{n(n+1)}{2}} \xhookrightarrow{} S(n)$ be the row-major injection of the upper triangular part into a symmetric matrix and let $\iota^{-1} : S(n) \twoheadrightarrow \mathbb{R}^{\frac{n(n+1)}{2}}$ be its inverse. Given a neural network $n_i : \mathbb{R}^{\frac{n(n+1)}{2}} \rightarrow \mathbb{R}^{\frac{n(n+1)}{2}}$ and an $X \in SPD(n)$, we may set:
\begin{equation}
\ell_i(X) =  \iota (n_i (\iota^{-1}(X))).
\end{equation}

Note that there is no longer any need for a projection to symmetric matrices, since we incorporate this structure directly into our vector field design. Moreover note that since $T_X SPD(n) \cong S(n) \cong \mathbb{R}^{\frac{n(n+1)}{2}}$, this vector field design is maximally expressive while being maximally efficient (representationally).

\textbf{B.2.3 \:\: Design 3: Parsimonious}

Although Design 2 is maximally expressive and efficient, in some cases where expressivity is less of a concern we may want a a reasonable parsimonious vector field design. Our answer to this is to directly parameterize a symmetric matrix via its upper triangular portion. To be explicit, let our vector field be parameterized by euclidean parameters $v \in \mathbb{R}^{\frac{n(n+1)}{2}}$ and, for $X \in SPD(n)$, be given by:
\begin{equation}
\ell_i(X) = \iota(v)
\end{equation}

This is a learnable vector field induced by a single tangent vector. Although highly efficient, its location-agnosticism makes it highly inexpressive.

\textbf{B.2.4 \:\: Design 4: Parsimonious Spectral}

One may also consider exploiting manifold-specific structure in the context of Design 3 to produce a more expressive vector field that remains fairly efficient parametrically. A vector field design that accomplishes this is one that allows a map from the spectrum of the local SPD matrix to the spectrum of the symmetric matrix in the vector field construction. We let $\text{spec}: SPD(n) \rightarrow \mathbb{R}^n$ be the spectral map that takes SPD matrices to a vector of their eigenvalues, sorted in descending order. To be explicit, let our vector field be parameterized by $Q \in O(n)$\footnote{$O(n)$ is the group of orthogonal matrices.}, a neural network $f_i : \mathbb{R}^n \rightarrow \mathbb{R}^n$, and, for $X \in SPD(n)$, be given by:
\begin{equation}
\ell_i(X) = Q \text{diag}(f_i(\text{spec}(X))) Q^T
\end{equation}

where $\text{diag}: \mathbb{R}^n \rightarrow \mathbb{R}^{n \times n}$ is the diagonal injection map. Observe that the spectrum of the symmetric matrix now depends locally on $X$, allowing for considerably more expressivity than in Design 3 at the cost of a low-dimensional neural network map $f_i: \mathbb{R}^n \rightarrow \mathbb{R}^n$. Moreover, the orthogonal constraint on $P$ may be preserved throughout optimization via one of a variety of easy-to-implement methods \citep{Casado2019CheapOC, Anil2019SortingOL}.

Design 1 is naive, but very inefficient. Design 2 exploits manifold structure to be maximally efficient while being maximally expressive. Design 3 showcases the other extreme (relative to Design 1) and gives a maximally parsimonious vector field construction. Design 4 showcases a more flexible version of Design 3 that allows for considerably greater learning capability\footnote{Verified empirically.} while still being representationally efficient. The purpose of describing these designs is to underscore the trade-off between expressivity and parameter-efficiency in designing parameterized vector fields (Designs 1 and 2 vs. Designs 3 and 4) as well as the need to utilize manifold-specific structure to obtain a maximally expressive and efficient vector field design (Design 1 vs. Design 2). Additionally, we highlight that expressivity for parameter-constrained vector field designs can be nontrivially increased with insignificant overhead via the introduction of manifold-specific dependencies (Design 3 vs. Design 4).

\subsection{Vector Field Design for Spherical Space}

For the spherical vector field design, we again apply the general design construction referenced at the start of Appendix~\ref{sec:vecfielddesign}. 
Similar to $\mathbb{H}^n$, $\mathbb{S}^n$ is an $n$-dimensional manifold which we treat as embedded in $\mathbb{R}^{n+1}$. Hence we parameterize a neural network $n_i : \mathbb{R}^{n+1} \rightarrow \mathbb{R}^{n+1}$ and set
\begin{equation}
\ell_i(x) = \text{proj}_{T_x \mathbb{S}^n} (n_i (x))
\end{equation}
to obtain our neural network-parameterized vector fields. As in the hyperbolic case, this vector field design is efficient and expressive, since $T_x \mathbb{S}^n \cong \mathbb{R}^{n}$.

\subsection{Feature Map-induced Vector Fields for General Manifolds}
\label{sec:featuremapgeneralman}
There is no perfect analog of a hyperplane for general manifolds. Hence, there is no immediately natural feature map in the general case. Despite this, we attempt to present a reasonable analog to hyperplane projection that extends to general manifolds. In particular, for a geodesically complete\footnote{A manifold $\M$ is said to be geodesically complete if any geodesic can be followed indefinitely \cite{Lee1997RiemannianMA}.} manifold $\M$, consider specifying a pseudo-hyperplane by a point $p \in \M$ and a non-zero vector $v \in T_p \M \setminus \{\mathbf{0}\}$ whose orthogonal complement we exponentiate at the base point $p$ to give the following definition:

\begin{equation}
h_{p, v} = \exp_p (\{w \in T_p \M | w^T v = 0 \})
\end{equation}

This definition\footnote{This notion was originally introduced in the context of hyperbolic space in \citet{Ungar2009AGS}.} has the benefit of reducing to the usual Euclidean hyperplane definition when the manifold under consideration is $\mathbb{R}^n$. However, this hyperplane definition is not particularly suitable for general manifolds since it assumes geodesic completeness, which may not hold. Here we propose an alternative general definition of a hyperplane that exponentiates the intersection of a local orthogonal complement with a closed ball of radius $r$, $\bar{B}_r(0) \subset T_p \M$, given below: 

\begin{equation}
h_{p,v,r} = \exp_p (\bar{B}_r(\mathbf{0}) \cap \{w \in T_p \M | w^T v = 0 \})
\end{equation}

Notice that this $h_{p,v,r}$ pseudo-hyperplane is a strict generalization of $h_{p,v}$ that does not require geodesic completeness (since $r$ is finite), and that in the limit as $r \rightarrow \infty$ we recover $h_{p,v}$.

A general feature map can then be defined by projecting to such a pseudo-hyperplane: 

\begin{equation}
g_{p,v,r}(x) = \min_{y \in h_{p,v,r}} d_\M (x,y)
\end{equation}

where $d_\M$ is the induced geodesic distance on $\M$.


We test this general construction for hyperbolic space and compare it with the horosphere projection construction in Appendix \ref{sec:horoplane}. The general construction performs reasonably well, but does not perform as well as the horosphere projection we give in this section. A more natural and performant manifold-dependent map can frequently be obtained by carefully considering the particular structure of the manifold (e.g. the spectral projection we give for $SPD(n)$).


\section{Experimental Details}

\subsection*{Experiments on Hyperbolic Space}

\subsubsection{Datasets}
\label{hyperbolicdatasets}

We apply our hyperbolic RResNet to node classification and link prediction on four graph datasets with varying $\delta$-hyperbolicity. 

\textbf{Airport ($\delta = 1$).} Airport is a dataset consisting of $2236$ nodes where nodes represent airports and edges represent airline routes \citep{chami2019hyperbolic}. For node classification, each airport is given a label corresponding to the population of the country it is in. Each airport has a $4$-dimensional feature vector consisting of geographic information.

\textbf{Pubmed ($\delta = 3.5$) and CoRA ($\delta = 11$).} Pubmed and CoRA are both citation networks consisting of $2708$ and $19717$ nodes each \citep{Pubmed, Cora}. In citation networks, each node represents a paper and edges indicate a shared author between papers. Each node has a label consisting of what academic subareas the paper belongs to.

\textbf{Disease ($\delta = 0$).} Disease is a synthetic dataset generated by simulating the SIR disease spreading model \citep{chami2019hyperbolic}. Node labels for classification indicate whether a node was infected or not and node features indicate a particular node's susceptibility to the disease.

\subsubsection{Architectural and Training Details}
\label{sec:hyperbolicarch&traindetails}
All of our testing uses the Poincar\'e ball model \citep{nickel2017poincare} to represent hyperbolic space. 
We use a similar setup to \citet{chami2019hyperbolic} to test RResNet's performance on hyperbolic space. First, in order to reduce the parameter count, we use a linear layer from the input dimension to a lower dimension before using RResNet as an encoder. For link-prediction tasks we use a Fermi-Dirac decoder and for node-classification tasks we use a linear decoder \citep{chami2019hyperbolic}.

\definecolor{mydarkblue}{rgb}{0,0.08,0.65}

For our results using a feature map induced vector field, we take the projection onto a fixed number of horospheres. Each horosphere is randomly initialized with $\omega$ drawn uniformly from $\mathbb{S}^{n-1}$ and $b \sim \mathcal{N}(0,1)$. We pass the horocycle projections to a linear layer followed by a Euclidean nonlinearity (typically ReLU \citep{Nair2010RectifiedLU}). During the training of each network, $\omega$ and $b$ are optimized using the same optimizer as the rest of the network. For further details regarding implementation, please see the accompanying  \textcolor{mydarkblue}{\href{https://github.com/CUAI/Riemannian-Residual-Neural-Networks}{Github code}}.

Horosphere projections are not the only natural feature map one can use, one alternative we experimented with was using parametetrized real eigenfunctions of the hyperbolic Laplacian as feature maps but we were unable to achieve similar performance to horosphere projections (results were significantly worse).

We use $250$ horospheres for Disease, Airport, and CoRA and $50$ horospheres for Pubmed. Models were trained for $500$, $10000$, $5000$, and $5000$ epochs for Disease, Airport, Pubmed, and CoRA, respectively, with the Adam optimizer \citep{kingma2014adam}. All other hyperparameters, such as learning rate and weight decay, were determined using random search.

All experiments were run on a single NVIDIA Quadro RTX A6000 48GB GPU.






\subsubsection{Comparison Between Embedded and Horocycle-induced Vector Field Designs}
\label{sec:hypablation  }

\begin{table*}[h]
\centering
\renewcommand{\arraystretch}{1.5}
\scalebox{0.95}{
\begin{tabular}{ccccc}
\hline
\textbf{Dataset} & Disease $(\delta = 0)$ & Airport $(\delta = 1)$ & Pubmed $(\delta = 3.5)$ & CoRA $(\delta = 11)$ \\
\hline
\textbf{RResNet Embedded} & $75.0$\tiny$\pm 5$ & $83.0$\tiny$\pm 2.0$ & $\mathbf{73.2}$\tiny$\pm 1.0$ & $\mathbf{59.6}$\tiny$\pm 1.0$ \\
\textbf{RResNet Horo} & $\mathbf{76.8}$\tiny$\pm 2.0$ & $\mathbf{96.9}$\tiny$\pm 0.3$ & $71.4$\tiny$\pm 2.2$ & $52.4$\tiny$\pm 5.5$ \\
\hline
\end{tabular}}
\caption{Node classification results for RResNet with two different vector field designs (test F1 score is the metric given).}
\label{tab:hypevecfieldablation}
\end{table*}

In order to investigate the effect vector field design has, we look at the performance of RResNet when using the embedded or horosphere projection-induced vector field in Table \ref{tab:hypevecfieldablation}. On more hyperbolic datasets (Disease and Airport), the more geometrically principled design attains higher F1 scores. This effect is reversed on the less hyperbolic datasets (Pubmed and CoRA), indicating that a more geometrically principled vector field only helps when the data geometry is similar to the model geometry, as expected.

\subsubsection{Comparison Between Horocycle-induced and Pseudo-Hyperplane-induced Vector Field Designs}
\label{sec:horoplane}

\begin{table*}[h]
\centering
\renewcommand{\arraystretch}{1.5}
\scalebox{0.85}{
\begin{tabular}{ccccc}
\hline
\textbf{Dataset} & Disease $(\delta = 0)$ & Airport $(\delta = 1)$ & Pubmed $(\delta = 3.5)$ & CoRA $(\delta = 11)$ \\
\hline
\textbf{RResNet Horocycle} & $\mathbf{76.8}$\tiny$\pm 2.0$ & $\mathbf{96.9}$\tiny$\pm 0.3$ & $\mathbf{71.4}$\tiny$\pm 2.2$ & $\mathbf{52.4}$\tiny$\pm 5.5$ \\
\textbf{RResNet Pseudo-Hyperplane} & $\textbf{77.2}$\tiny$\pm 0.4$ & $90.3$\tiny$\pm 4.5$ & $66.7$\tiny$\pm 5.0$ & $41.4$\tiny$\pm 5.7$ \\
\hline
\end{tabular}}
\caption{Node classification results for RResNet with two different vector field designs (test F1 score is the metric given).}
\label{tab:horoplaneresults}
\end{table*}

In Table \ref{tab:horoplaneresults} we compare the RResNet construction with vector fields induced by projection to pseudo-hyperplanes (as defined in the main paper in Section \ref{sec:featuremap}) for hyperbolic space (RResNet Pseudo-Hyperplane) to the horocycle projection-induced vector field RResNet construction (RResNet Horocycle). Note that RResNet Pseudo-Hyperplane performs worse for most tasks, although the construction is more general (as mentioned in the main paper).


\subsubsection{Ablation Study}
\label{sec:hypablation}
\textbf{Nonlinearity Ablation}

\begin{table*}[h]
\centering
\renewcommand{\arraystretch}{1.5}
\scalebox{0.85}{
\begin{tabular}{ccccc}
\hline
\textbf{Dataset} & Disease $(\delta = 0)$ & Airport $(\delta = 1)$ & Pubmed $(\delta = 3.5)$ & CoRA $(\delta = 11)$ \\
\hline
\textbf{RResNet Horo w/o Nonlinearity} & $71.9$\tiny$\pm 9.2$ & $\mathbf{96.9}$\tiny$\pm 3.0$ & $\mathbf{71.2}$\tiny$\pm 1.1$ & $49.6$\tiny$\pm 2.0$ \\
\textbf{RResNet Horo} & $\mathbf{76.8}$\tiny$\pm 2.0$ & $\mathbf{96.9}$\tiny$\pm 0.3$ & $\mathbf{71.4}$\tiny$\pm 2.2$ & $\mathbf{52.4}$\tiny$\pm 5.5$ \\
\hline
\end{tabular}}
\caption{Node classification results for RResNet with and without a nonlinearity between layers (test F1 score is the metric given).}
\label{tab:hypevecfieldablation}
\end{table*}

To study the expressiveness of the horocycle induced vector field design, we ablate the nonlinearity in the vector field. With the nonlinearity, the F1 score either increases or remains the same across all datasets, which the advantage being most pronounced for Disease.

\textbf{Geometry Ablation}

\begin{table*}[h]
    \centering
    \renewcommand{\arraystretch}{1.5}
    \begin{tabular}{cc}
        \hline
        \textbf{Dataset} & Disease $(\delta = 0)$  \\ 
        \hline 
        RResNet Embedded (Euclidean) & $67.3$\tiny$\pm 21.0$ \\
        \textbf{RResNet Embedded (Hyperbolic)} & $75.0$\tiny$\pm 5.0$ \\
        RResNet Feature Map (Euclidean) & $73.1$\tiny$\pm 3.4$  \\ 
        \textbf{RResNet Feature Map (Hyperbolic)} & $\mathbf{76.8}$\tiny$\pm 2.0$  \\ 
        \hline
    \end{tabular}
    \caption{Node classification results of RResNet with different vector field designs and model geometry (test F1 score is the metric given). When swapping geometry for a specific model, all hyperparameters are kept the same, which we are able to do easily with our architecture.}
    \label{tab:hypgeoablation}
\end{table*}

We look at the performance of varying RResNets on the most hyperbolic dataset to identify the effect model geometry has in Table \ref{tab:hypgeoablation}. As expected, using hyperbolic space yields higher F1 scores with lower standard deviations. In particular, the high standard deviation of $21.0$ for ``RResNet Embedded (Euclidean)" indicates that it fails to properly learn in a number of trials.

\textbf{Residual Connection Ablation}

It is reasonable to try other residual connection implementations outside of our natural geometric Riemannian exp-map based implementation. In particular, one may try to implement a Riemannian residual neural networks directly via a gyrovector \cite{Ungar2009AGS} addition. We give the results in Table \ref{tab:hypvecfieldres} and show that not only is this method less desirable geometrically, but also gives worse results on our chosen benchmarks. The Euclidean model is given as a baseline and the Riemannian ResNet here is a reference. All models are implemented with a comparable number of parameters and are two layer residual neural networks.

\begin{table*}[htb!]
\centering
\renewcommand{\arraystretch}{1.5}
\begin{tabular}{cc}
\hline
\textbf{Dataset} & Airport $(\delta = 1)$ \\
\hline
\textbf{Euclidean} & $69.4$\tiny$\pm 1.8$ \\
\textbf{Gyrovector} & $60.8$\tiny$\pm 0.9$ \\
\textbf{RResNet Horo} & $\textbf{75.9}$\tiny$\pm 2.5$ \\
\hline
\end{tabular}
\caption{Node classification results for RResNet with three different residual connection designs (test F1 score is the metric given).}
\label{tab:hypvecfieldres}
\end{table*}



\subsection*{Experiments on the SPD Manifold}
\subsubsection{Datasets}
\label{sec:spd_appendix}

We apply our SPD architecture on four different video recognition tasks. For all tasks, we generate covariance or correlation matrices sampled from each video's frames. Given frames $t \in \{1, \dots, T\}$ and their corresponding feature vectors $x_t \in \R^n$, we generate a $n\times n$ covariance matrix by sampling the frames: $\text{X} = \frac{1}{T-1} \sum_{t=1}^T (x_t - \mu)(x_t - \mu)^T$. Optionally, we can divide the matrices by the standard deviations to instead generate correlation matrices. For certain tasks, we find that these have better conditioning. 

While covariance and correlation matrices are positive semi-definite, they are not necessarily SPD. In fact, they are only SPD if the set of sampled vectors, $\{x_1, \dots, x_T\}$, consists of $n$ linearly-independent vectors. If the sampled vectors $x_t$, $x_{t+1}$, are similar, which is the case for neighboring frames of a video, the matrices may be close to singular. This phenomenon poses issues in downstream tasks such as taking a matrix logarithm, which can create numerical instability. For all tasks, we preprocess our data by removing covariance matrices which fail a Cholesky decomposition.

\textbf{AFEW.} AFEW \citep{Dhall2011StaticFE} is an emotion recognition dataset consisting of 1,345 videos and 7 classes. As done in \citet{Huang2017ARN, Brooks2019RiemannianBN}, we use covariance matrices created from $20\times 20$ video frames, flattened into $400$-dimensional $x_t$ vectors. 

\textbf{FPHA.} The First-Person Hand Action Benchmark (FPHA) \citep{GarciaHernando2018FirstPersonHA} consists of 1,175 videos of humans performing 45 different tasks. The dataset includes the $(x, y, z)$ coordinates of 21 joint locations from a human hand. Following the approach of \citet{Hussein2013HumanAR}, for each frame, we flatten the coordinates into a 63-dimensional vector $x_t$. We then take the correlation matrices. We use subjects 1-3 for training and 4-6 for validation. 

\textbf{NTU RGB+D.} NTU RGB+D \citep{shahroudy2016ntu} is an action recognition dataset which includes the 3D locations of 25 body joints. NTU RGB+D is a large scale dataset with 56,880 videos and 60 tasks. For our $x_t$ vectors, we use the flattened versions of 3D joint coordinates as feature vectors. Our matrices have dimension 75.  

\textbf{HDM05.} Mocap Database HDM05 \citep{HDM05} is another action recognition dataset which includes 3D locations of 31 joints. Following the task designed in \citet{Huang2017ARN}, the goal is to classify each video clip into one of 117 action classes. We use the covariance matrices provided in \citet{Brooks2019RiemannianBN}.
 \begin{figure}[h]
\centering
\includegraphics[width=0.48\textwidth]{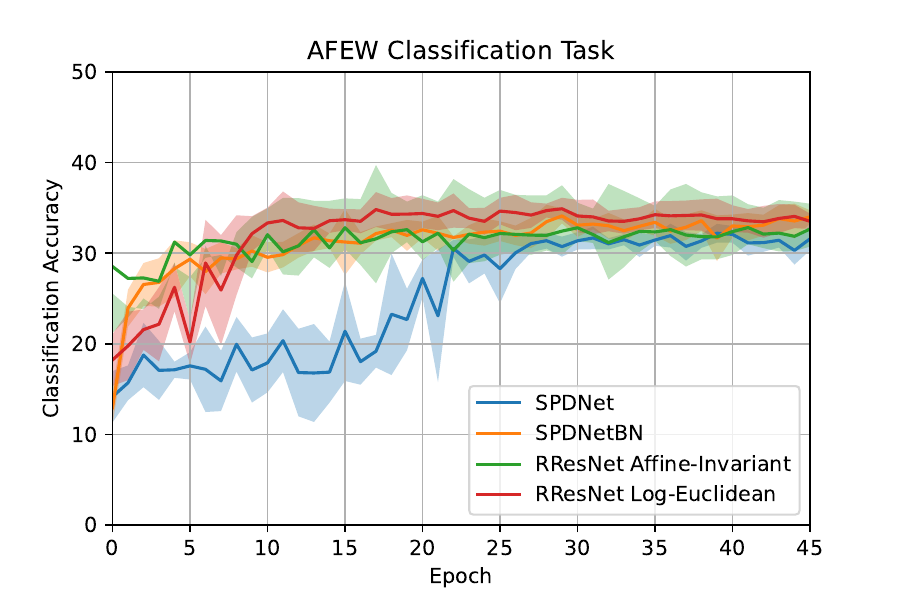}
\includegraphics[width=0.48\textwidth]{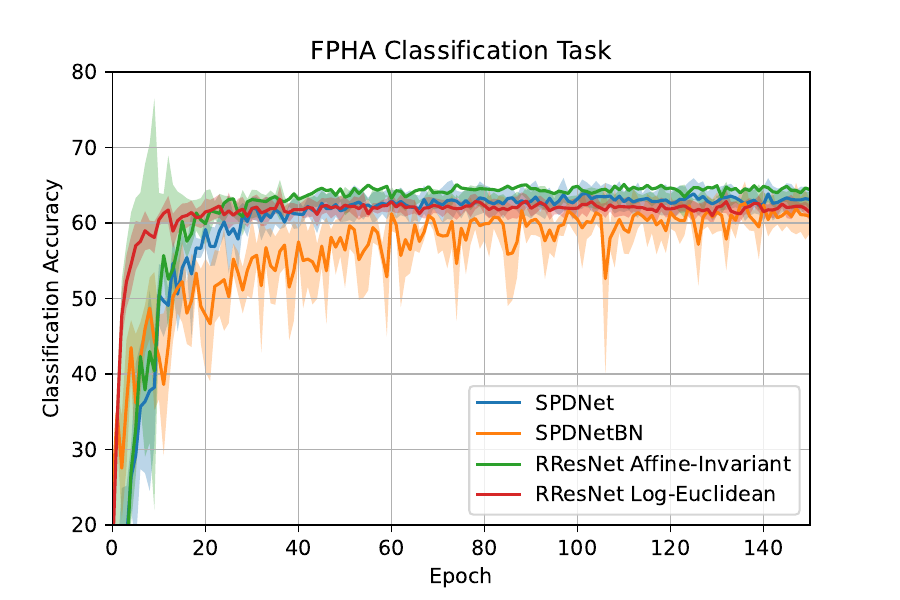}
\includegraphics[width=0.48\textwidth]{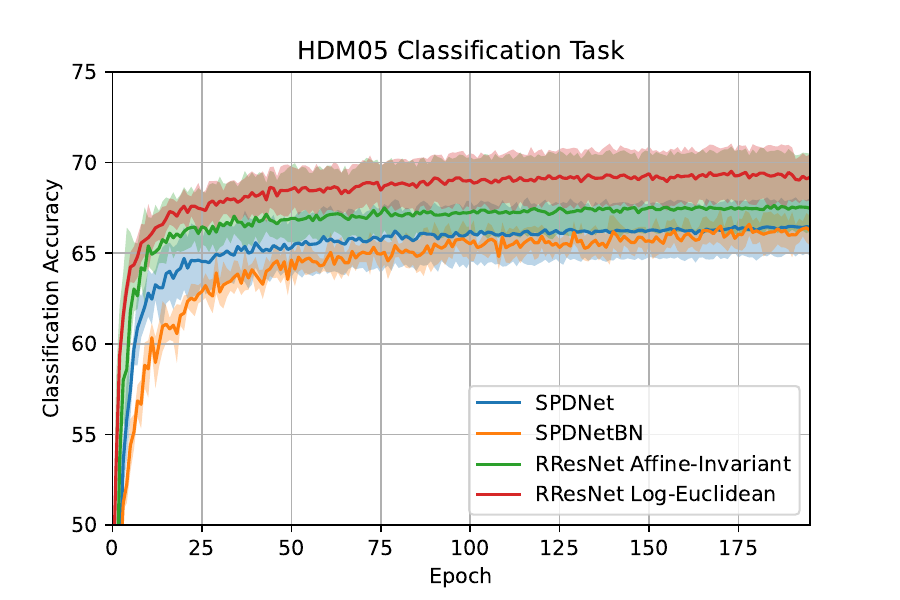}
\includegraphics[width=0.48\textwidth]{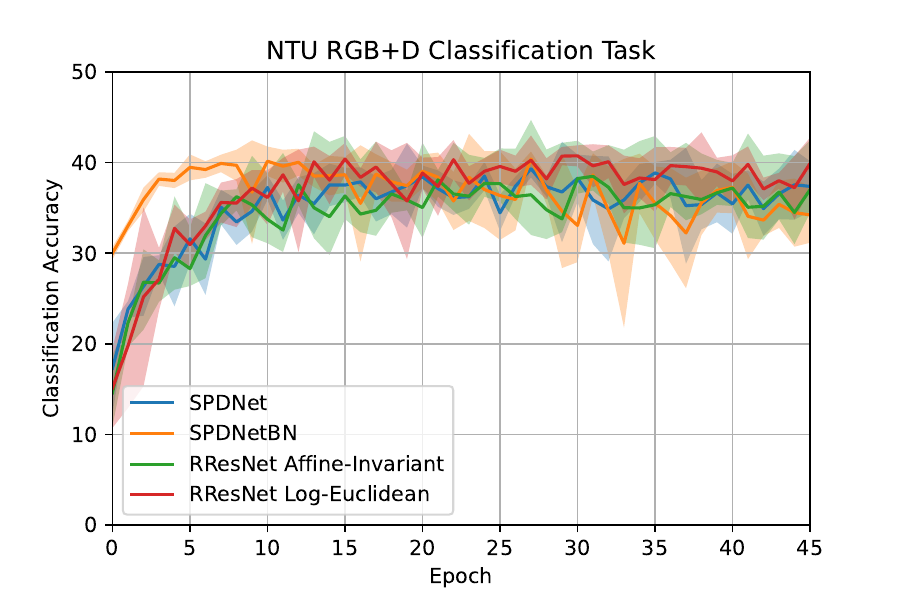}

\caption{Validation accuracies for our RResNet compared to the SPDNet \citep{Huang2017ARN} and SPDNetBN \citep{Brooks2019RiemannianBN} baselines. For each model, results are averaged over five trials. Error bars denote one standard deviation away from the mean accuracy. We observe that our model converges faster and achieves higher accuracies than SPDNet and SPDNetBN.}
\label{fig:spd_training}
\end{figure}

\subsubsection{Architectural Details}
\label{sec:spd_arch}

Given a dataset of covariance matrices, our goal is to classify a matrix into one of several classes. To illustrate, we give our architecture for the AFEW task as an example. Because of how costly it would be to parameterize vector fields at this dimension, we use a BiMap layer \citep{Huang2017ARN}, $\text{BiMap}^{d_{i}}_{d_{i+1}}: SPD(d_{i}) \rightarrow SPD(d_{i+1})$ as a base point remapping from $400 \times 400$ matrices to $50 \times 50$ matrices. We use vector field design 4 from Appendix \ref{sec:vecfielddesign}. In the context of this problem, we have:
\begin{equation}
\ell_1(X) = Q \text{diag}(f_1(\text{spec}(X))) Q^T
\end{equation}

where $f_1: \mathbb{R}^{50} \rightarrow \mathbb{R}^{50}$, $\text{spec} : SPD(50) \rightarrow \mathbb{R}^{50}$, $P \in O(50)$ (spec is defined above in Appendix \ref{sec:vecfielddesign}). In practice, we experiment with a variety of $f_1$ designs, such as sequences of linear layers or 1D convolutions. Note the vector field is a map $\ell_1: SPD(50) \rightarrow T\:SPD(50)$. We express our forward pass as

\begin{align}
    g(x) = \exp_{\text{BiMap}^{400}_{50}(x)}(\ell_1(\text{BiMap}^{400}_{50}(x))) 
\end{align} 

which is a map $g : SPD(400) \rightarrow SPD(50)$. Our $\exp$ map depends on the Riemannian metric we choose on the manifold. Thereafter we apply a logarithm to the eigenvalues of the $50 \times 50$ matrices (this helps linearize features \cite{Brooks2019RiemannianBN}). Lastly we flatten the matrices and use a linear map from dimension $2500$ to dimension $7$ (representing the $7$ different emotions). We use a simple cross entropy loss \citep{Goodfellow2015DeepL} to train the model. 

\subsubsection{Results}
 We compare our RResNet design above (Appendix \ref{sec:spd_arch}) to SPDNet \citep{Brooks2019RiemannianBN,Huang2017ARN}, a network architecture for SPD matrix learning. All models have a comparable number of parameters. To replicate the results of \citet{Brooks2019RiemannianBN}, we use a learning rate of $5\cdot 10^{-2}$ for the baseline. We find that any higher learning rate causes training instability. However, we observe that our model remains stable with a learning rate of  $1\cdot 10^{-1}$. Our model has faster convergence and achieves a higher accuracy than SPDNet and SPDNetBN (see Table~\ref{tab:spd_table} in the main paper and Figure~\ref{fig:spd_training} above). Moreover, our model's ability to switch out geometries (as given by the log-Euclidean and affine-invariant metrics) gives the ability to outperform prior work on all tasks. 

\subsubsection{Comparison Between Vector Field Designs}
For the SPD manifold, we illustrate differences between our four different vector field designs outlined in Appendix \ref{sec:vecfielddesign} on the AFEW task. Results are given in Table \ref{tab:spd_ablation}. Note that the chosen spectral map-induced vector field is very efficient in terms of parameter count and performs best in terms of accuracy.
\begin{table*}[h]
\centering
\scalebox{1}{
\begin{tabular}{lcr}
\toprule
& AFEW\citep{Dhall2011StaticFE} & Number of Parameters\\ 
\midrule
 Naive & $34.90$\tiny$\pm 0.74$& 6,290,007 \\
Structured & $34.76 $\tiny$\pm 1.07$& 1,664,407\\
Parsimonious & $34.82$ \tiny $\pm 1.82$&38,782\\
\textbf{Spectral Map}&$\mathbf{36.38} $\tiny$\mathbf{\pm 1.29}$& 45,057 \\
\bottomrule
\end{tabular}}
\caption{We compare the accuracy of our four vector field designs for the SPD manifold. We see that the spectral map provides the best balance of accuracy and parameter efficiency.}
\label{tab:spd_ablation}
\end{table*}

\subsubsection{Ablation Study}
\label{sec:spdablation}

\textbf{Nonlinearity Ablation}

 We ablate the nonlinearity in the spectral map design in Table~\ref{tab:spd_nonlin}, and find that the nonlinearity slightly improves performance. For AFEW, we use a one layer vector field, which is why the reported accuracies are the same.
\begin{table}[h]
\centering
\scalebox{1}{
\begin{tabular}{lccccc}
\toprule
& AFEW & FPHA & NTU RGB+D\ & HDM05\\ 
\midrule
Aff-Inv w/o Nonlinearity & $35.17$\tiny$\pm 1.78$& $65.03$\tiny$ \pm 2.22$ & $41.27$\tiny$ \pm 0.22$ &  $65.92 $\tiny$\pm 1.27$\\
Aff-Inv & $35.17$ \tiny $\pm 1.78$&$\mathbf{66.53} $\tiny$\mathbf{\pm 1.64}$ & $41.00$ \tiny $\pm 0.50$& $67.91$ \tiny${\pm 0.66}$\\
Log-Euc w/o Nonlinearity & $36.38$\tiny$\pm 1.29$& $65.25 $\tiny$\pm 2.14$ & $42.87$\tiny$\pm 0.83$ & $68.73 $\tiny$\pm 1.75$\\

Log-Euc&$\mathbf{36.38}$\tiny $\mathbf{\pm 1.29}$& $64.58 $\tiny$\pm 0.98$& $\mathbf{42.99 }$\tiny$\mathbf{\pm 0.23}$& $\mathbf{69.80}$\tiny$\mathbf{ \pm 1.51}$ \\
\bottomrule
\end{tabular}}

\caption{We show that classification accuracy either improves or remains the same with the nonlinearity. For AFEW, we use one layer in the vector field, which is why the reported accuracies are the same.}
\label{tab:spd_nonlin}
\end{table}

\textbf{Geometry Ablation}

We study the geometry of the SPD manifold by comparing our Riemannian ResNet to a Euclidean ResNet. For the Euclidean network, we treat each matrix as a Euclidean vector by flattening it into a length $n\times n$ vector. We then pass it through a Euclidean ResNet. Our results in Table~\ref{tab:eucspd} show that the Riemannian ResNets (Aff-Inv and Log-Euc) perform significantly better across all datasets. 

\begin{table}[h]
\centering
\scalebox{1}{
\begin{tabular}{lccccc}
\toprule
& AFEW & FPHA & NTU RGB+D\ & HDM05\\ 
\midrule
Euclidean & $30.08$\tiny$\pm 1.36$& $30.72 $\tiny$\pm 1.03$ & $34.63$\tiny$\pm 3.10$ & $0.80 $\tiny$\pm 0.10$\\
Aff-Inv & $35.17$ \tiny $\pm 1.78$&$\mathbf{66.53} $\tiny$\mathbf{\pm 1.64}$ & $41.00$ \tiny $\pm 0.50$& $67.91$ \tiny${\pm 1.27}$\\
Log-Euc&$\mathbf{36.38}$\tiny $\mathbf{\pm 0.24}$& $64.58 $\tiny$\pm 0.98$& $\mathbf{42.99 }$\tiny$\mathbf{\pm 0.23}$& $\mathbf{69.80}$\tiny$\mathbf{ \pm 1.51}$ \\
\bottomrule
\end{tabular}}

\caption{We show that the Euclidean ResNet performs worse across all datasets, and fails for HDM05.}
\label{tab:eucspd}
\vspace{-10px}
\end{table}

\textbf{Residual Connection Ablation}

Similar to the gyrocalculus used in \citet{ganea2018hyperbolic},  \citet{Lopez2021VectorvaluedDA} have extended gyrovector operations to the manifold of SPD matrices. In particular, the authors define M\"obius addition as $X \oplus Y = \sqrt{X}Y\sqrt{X}$ for SPD matrices $X, Y$. It is reasonable to ask how this purely algebraic, non-geometric construct performs when used to implement a residual connection. With this choice of addition, the residual connection for a ResNet specific to the SPD manifold would have the form $\ell_i(X) + X = \sqrt{\ell_i(X)}X\sqrt{\ell_i(X)}$. In Table~\ref{tab:gyrospd}, we show that this choice of addition struggles to reach the accuracy of our Riemannian ResNet design.

\begin{table}[h]
\centering
\scalebox{1}{
\begin{tabular}{lccccc}
\toprule
& AFEW & FPHA & NTU RGB+D\ & HDM05\\ 
\midrule
Gyrovector & $23.23$\tiny$\pm 0.98$& $61.33 $\tiny$\pm 4.74$ & $40.77$\tiny$\pm 3.10$ & $5.69 $\tiny$\pm 2.15$\\
Aff-Inv & $35.17$ \tiny $\pm 1.78$&$\mathbf{66.53} $\tiny$\mathbf{\pm 1.64}$ & $41.00$ \tiny $\pm 0.50$& $67.91$ \tiny${\pm 1.27}$\\
Log-Euc&$\mathbf{36.38}$\tiny $\mathbf{\pm 1.29}$& $64.58 $\tiny$\pm 0.98$& $\mathbf{42.99 }$\tiny$\mathbf{\pm 0.23}$& $\mathbf{69.80}$\tiny$\mathbf{ \pm 1.51}$ \\
\bottomrule
\end{tabular}}

\caption{We show that the Riemannian ResNet model with M\"obius addition struggles to reach the classification accuracies of our exponential map design. The difference is most pronounced on HDM05, where the gyrovector model struggles to learn meaningful representations.}
\label{tab:gyrospd}
\end{table}

\subsection*{Experiments on Spherical Space}

\subsubsection{Dataset}

We wish to explore the generality of our method: in particular, our ability to vary geometry without constructing entirely new operations for each manifold. 
We repeat one of the experiments tested on our hyperbolic RResNet, swapping out the hyperbolic manifold for the spherical manifold.

\textbf{CoRA.} This dataset is described above in Appendix~\ref{hyperbolicdatasets}. With $\delta=11$, CoRA is the least hyperbolic of the datasets tested with our hyperbolic RResNet. As such, we wanted to try swapping the RResNet geometry to better match the data geometry. 

\subsubsection{Architectural Details}

The design of our spherical RResNet is identical to that of our hyperbolic RResNet (described in Appendix~\ref{sec:hyperbolicarch&traindetails}), aside from switching the geometric representation from hyperbolic to spherical. As before, we first have a linear layer to move from the input dimension to a lower dimension. Then we use our RResNet as an encoder. Here we only test link prediction, so we use a Fermi-Dirac decoder.

We train for 2000 epochs using the Adam optimizer  \citep{kingma2014adam}, and we again found all hyperparameters via random search.

\subsubsection{Results}

We give results for link prediction on CoRA, displayed in Table~\ref{tab:sphereresults}. Mean and standard deviation across 5 separate trials are reported. 

\begin{table*}[htb!]
\centering
\scalebox{1.0}{
\begin{tabular}{clc}
\toprule
& \textbf{Dataset} & CoRA \\
& \textbf{Hyperbolicity} & $\delta = 11$ \\
\midrule
\parbox[t]{2mm}{\multirow{2}{*}{\rotatebox[origin=c]{90}{\scriptsize{Hyp}}}}
& RResNet & $88.9$\tiny$\pm 0.2$ \\
& RResNet Graph & $87.6$\tiny$\pm 0.9$ \\ 
\midrule
\parbox[t]{2mm}{\multirow{2}{*}{\rotatebox[origin=c]{90}{\scriptsize{Sphere}}}}
& RResNet & $90.7$\tiny$\pm 1.0$ \\
& RResNet Graph & $\mathbf{91.7}$\tiny$\pm 0.4$ \\
\bottomrule
\end{tabular}}
\caption{Test accuracy of various models, in terms of ROC AUC.}
\label{tab:sphereresults}
\end{table*}

We find that even the most basic spherical RResNet design, which does not use a feature map, outperforms both hyperbolic RResNets. This indicates that our model improves when endowed with geometry more suitable for given data. Additionally, our model's flexibility allows us to easily obtain such results without altering the architecture.

\section{Theoretical Results}
\label{sec:theory}

In this section we give a variety of theoretical results that demonstrate the principled nature of our Riemannian ResNet construction.

\subsection{Reduction to Standard ResNet in Euclidean Case}
We show that our construction agrees with the standard ResNet when the underlying manifold is Euclidean space and when we are using the embedded vector field design. This aligns with our intuition and shows that our construction is a natural generalization of previous work.

\begin{prop}
    When $\M^{(i)} \cong \R^{d_i}$, our RResNet with the embedded vector field design is a standard residual network.
\end{prop}

\begin{proof}
    Note that the embedded vector fields take the form:
    \begin{equation}
    \ell_i(x) = \text{proj}_{T_x \mathbb{R}^n}(n_i(x)) = n_i(x)
    \end{equation}
    
    for a parameterized neural network $n_i : \mathbb{R}^{d_{i-1}} \rightarrow \mathbb{R}^{d_i}$, meaning that our $\ell_i$ are standard neural networks. The $h_i : \mathbb{R}^{d_{i-1}} \rightarrow \mathbb{R}^{d_i}$ can be replaced by Euclidean linear layers that go from dimension $d_{i-1}$ to dimension $d_i$. Also observe since $\exp_x(v) = x + v$, our neural network construction becomes:
    \begin{align}
        f(x) &= x^{(m)}\\
        x^{(0)} &= x\\
        x^{(i)} &= \exp_{h_i(x^{(i - 1)})}(\ell_i(h_i(x^{(i - 1)}))) \\
        &= h_i(x^{(i - 1)}) + \ell_i(h_i(x^{(i - 1)})) \\
        &= h_i(x^{(i - 1)}) + n_i(h_i(x^{(i - 1)}))
    \end{align}
    
    where the last equality holds $\forall i \in [m]$. Moreover, if all $d_i$ are the same, we can use the identity map for our $h_i$, implying:
    \begin{equation}
    x^{(i)} = x^{(i - 1)} + n_i(x^{(i - 1)}) \ \forall i \in [m]
    \end{equation}
    
    Hence our neural network architecture reduces precisely to that of Euclidean residual neural networks. 
\end{proof}

\subsection{Hyperbolic Neural Networks (HNNs) \citep{ganea2018hyperbolic} Learn via a Hyperbolic Bias}
\label{sec:hnnred}

We make note of the fact that although the gyrovector generalization of Euclidean networks offered by \citet{ganea2018hyperbolic} is algebraic and generalizable to many manifolds such as hyperbolic space and the manifold of SPD matrices \citep{Lopez2021VectorvaluedDA}, the linear layer of the construction is Euclidean, except for the hyperbolic bias addition. We illustrate this in what follows.

\begin{prop}
For $x \in \mathbb{H}^n$ and hyperbolic matrix-vector multiplication \citep{ganea2018hyperbolic} defined by
\begin{equation}
M^\otimes (x) = \tanh\left( \frac{||Mx||}{||x||} \tanh^{-1} (||x||) \right) \frac{Mx}{||Mx||}
\end{equation}
where $M : \mathbb{R}^n \rightarrow \mathbb{R}^n$ is a linear map, we have
\begin{equation}
M_2^\otimes(M_1^\otimes (x)) = \tanh\left(\frac{||M_2||||M_1x||}{||x||} \tanh^{-1} (||x||) \right) \frac{M_2M_1x}{||M_2|| ||M_1x||} = (M_2M_1)^{\otimes} (x)
\end{equation}
\end{prop}

\begin{proof}
For two linear maps of the same size $M_1, M_2$ we have:
\begin{equation}
M_2^\otimes(M_1^\otimes (x)) =  \tanh\left( \frac{||M_2M_1^\otimes (x)||}{||M_1^\otimes (x)||} \tanh^{-1} (||M_1^\otimes (x)||) \right) \frac{M_2M_1^\otimes (x)}{||M_2M_1^\otimes (x)||}
\end{equation}
\begin{equation}
= \tanh\left( \frac{||M_2M_1^\otimes (x)||}{\tanh\left( \frac{||M_1x||}{||x||} \tanh^{-1} (||x||) \right)} \left( \frac{||M_1x||}{||x||} \tanh^{-1} (||x||) \right) \right) \frac{M_2M_1^\otimes (x)}{||M_2M_1^\otimes (x)||}
\end{equation}
\begin{equation}
= \tanh\left(\frac{||M_2||||M_1x||}{||x||} \tanh^{-1} (||x||) \right) \frac{M_2M_1x}{||M_2|| ||M_1x||} =  (M_2M_1)^{\otimes} (x)
\end{equation}
\end{proof}

We see that we have cancellation that de facto reduces the learning of two hyperbolic linear layers with no hyperbolic bias to the learning of a single hyperbolic layer. Inductively, this precise argument applies to any number of layers. This reduction is characteristic to what one sees in the case of Euclidean networks, and more importantly, from the above equation we see that learning hyperbolic linear layers de facto reduces to learning Euclidean linear maps ($M_1$ and $M_2$ above) that are placed in between an initial Riemannian $\log$ map (taken at the origin) and a trailing Riemannian $\exp$ map (taken at the origin).

Thus, the main non-Euclidean, hyperbolic construct in \citet{ganea2018hyperbolic} is the hyperbolic bias, introduced in Section 3.2 of \citet{ganea2018hyperbolic}. Our method is distinctly different in that even simple residual linear layers make use of geodesic information; hence, learning does not reduce to the Euclidean case.


\section{Comparison with Other Constructions}
\label{sec:comparison_construction}


Here we elaborate on how our method compares with other constructions, elucidating a claim made in the main paper. We note that compared to other methods, our construction is fully general (in the sense that it extends to all Riemannian manifolds) and better conforms with geometry. For example, general methods like HNN \citep{ganea2018hyperbolic}, HGCN \cite{chami2019hyperbolic}, and SPDNetBN \citep{Brooks2019RiemannianBN} use the fact that hyperbolic space and the SPD manifold are spaces with everywhere non-negative curvature, meaning that geodesics are unique. As such, core building blocks of these models globally project to a Euclidean space via a map known as the Riemannian $\log$ map, which can be thought of as an inverse to the exponential map. This global projection relies on the choice of a particular base point (equivalent to the base point in the $\exp$ map definition), which is arbitrarily selected. Once this projection has taken place, prior methods usually simply perform a Euclidean operation, and project back to the manifold via the usual Riemannian $\exp$. This system does not generalize to manifolds which are not globally diffeomorphic to Euclidean space (note this is quite restrictive and different from being locally diffeomorphic to Euclidean space), and furthermore, the reliance on fundamentally Euclidean operations and arbitrary base point destroys the geodesic geometry: the log map can be thought of as linearizing manifold geometry with respect to a certain base point---the further away points are from this base point, the more distorted the projected geometry.

More specialized methods like \citet{chen-etal-2022-fully, shimizu2020hyperbolic} instead work with algebraic operations that exist only for hyperbolic space. These do not generalize to arbitrary manifolds, which limits potential applications. By comparison, our method simultaneously generalizes to arbitrary manifolds and directly conforms with underlying geometry.

\end{document}